\documentclass[twoside]{article}
\usepackage[accepted]{aistats2021}

\usepackage[round]{natbib}

\usepackage{booktabs}
\usepackage{multirow}
\usepackage{enumitem}
\usepackage{algorithmic}
\usepackage[utf8]{inputenc} 
\usepackage[T1]{fontenc}    
\usepackage{hyperref}       
\usepackage{url}            
\usepackage{booktabs}       
\usepackage{amsfonts}       
\usepackage{nicefrac}       
\usepackage{microtype}      
\usepackage{graphicx}
\usepackage{float}

\usepackage{amsmath,amsthm,amssymb}
\usepackage{color}
\usepackage{comment}
\usepackage{cancel}

\usepackage{dsfont}
\usepackage[english]{babel}
\usepackage{bbm}
\usepackage{float}
\usepackage{trimclip}
\usepackage{wrapfig}
\newcommand{\mtrian}{\raisebox{-0.25ex}{\clipbox{0em 1.25ex 0em 0em}{$\triangleq$}}}
\newcommand{\trianglepropto}{\overset{\mtrian}{\propto}}

\usepackage[ruled,vlined]{algorithm2e}
\newtheorem{definition}{{Definition}}
\newtheorem{assumption}{{Assumption}} 
\newtheorem{theorem}{{Theorem}}
\newtheorem{proposition}{{Proposition}}

\newtheorem{assumption_SM}{{Assumption }} 

\newtheorem{proposition_SM}{{Proposition }}

\usepackage[dvipsnames]{xcolor}

\usepackage{footnote}
\makeatletter
\newcommand\footnoteref[1]{\protected@xdef\@thefnmark{\ref{#1}}\@footnotemark}
\makeatother

\usepackage{amsmath}
\usepackage{amsfonts}
\usepackage{amssymb}
\usepackage{wrapfig}
\usepackage{subcaption}
\usepackage{multirow}
 \usepackage{mathtools} 

\usepackage{verbatim}

\usepackage{anyfontsize}

\usepackage{color}
\usepackage{tikz}
\usetikzlibrary{arrows,shapes,snakes,automata,backgrounds,fit,petri}
\usepackage{adjustbox}

\makeatletter
\newcommand{\distas}[1]{\mathbin{\overset{#1}{\kern\z@\sim}}}%

\usepackage{enumitem}

\newcommand{\indep}{\perp \!\!\!\! \perp}

\newcommand{\beq}{\vspace{0mm}\begin{equation}}
\newcommand{\eeq}{\vspace{0mm}\end{equation}}
\newcommand{\beqs}{\vspace{0mm}\begin{eqnarray}}
\newcommand{\eeqs}{\vspace{0mm}\end{eqnarray}}
\newcommand{\barr}{\begin{array}}
\newcommand{\earr}{\end{array}}

\ifx\theorem\undefined
\newtheorem{theorem}{Theorem} 
\fi

\ifx\lemma\undefined
\newtheorem{lemma}{Lemma}
\fi

\ifx\proposition\undefined
\newtheorem{proposition}[theorem]{Proposition}
\fi

\ifx\corollary\undefined
\newtheorem{corollary}{Corollary}
\fi

\ifx\assumption\undefined
\newtheorem{assumption}{Assumption}
\fi

\ifx\definition\undefined
\newtheorem{definition}{Definition}
\fi

\ifx\remark\undefined
\newtheorem{remark}{Remark}
\fi

\begin{document}

%

%

\twocolumn[

\aistatstitle{Counterfactual Representation Learning with Balancing Weights}
\aistatsauthor{Serge Assaad$^1$ \And Shuxi Zeng$^2$ \And Chenyang Tao$^1$ \And Shounak Datta$^1$}
\aistatsauthor{Nikhil Mehta$^1$ \And Ricardo Henao$^1$ \And Fan Li$^2$ \And Lawrence Carin$^1$}
\aistatsaddress{ $^1$Department of ECE, Duke University \And $^2$Department of Statistical Science, Duke University}
\runningauthor{Assaad, Zeng, Tao, Datta, Mehta, Henao, Li, Carin}]
\begin{abstract}
A key to causal inference with observational data is achieving balance in predictive features associated with each treatment type. Recent literature has explored representation learning to achieve this goal. In this work, we discuss the pitfalls of these strategies -- such as a steep trade-off between achieving balance and predictive power -- and present a remedy via the integration of balancing weights in causal learning. Specifically, we theoretically link balance to the quality of propensity estimation, emphasize the importance of identifying a proper target population, and elaborate on the complementary roles of feature balancing and weight adjustments. Using these concepts, we then develop an algorithm for flexible, scalable and accurate estimation of causal effects. Finally, we show how the learned weighted representations may serve to facilitate alternative causal learning procedures with appealing statistical features. We conduct an extensive set of experiments on both synthetic examples and standard benchmarks, and report encouraging results relative to state-of-the-art baselines.
\end{abstract}
\section{INTRODUCTION}
Solving many scientific, engineering, and socioeconomic problems -- \textit{e.g.}, personalized healthcare \citep{CI_health,EHR}, computational advertising \citep{CI_ads}, and policymaking \citep{CI_policy} -- requires an understanding of cause and effect beyond observed associations. Consequently, the study of {\it causal inference} \citep{Pearl,Rubin_potential_outcomes} is central to various disciplines and has received growing attention in the machine learning community. To exploit the new opportunities and cope with the challenges brought by modern datasets, various new causal inference methods have been proposed  \citep{Shalit,GANITE,louizos2017causal,hassanpour2019counterfactual,johansson2018learning,johansson2020generalization,LiFu_matching,alaa18,alaa2017bayesian}.

This paper focuses on predicting conditional average treatment effects (CATE) from observational data, defined as the difference between an individual's expected potential outcomes for different treatment conditions. This problem differs fundamentally from standard supervised learning \citep{Pearl,Rubin_potential_outcomes}, because for each unit only the potential outcome corresponding to the assigned treatment is observed and the other potential outcome is missing. The absence of the ``counterfactual'' outcome prohibits the direct learning and validation of causal effects. Further, observational studies are subject to selection bias due to confounders \citep{SelectionBias} -- variables that affect both the treatment assignment and the outcomes. Within the associated data this is typically manifested as covariate imbalance \citep{Shalit}, {\it i.e.}, treatment-dependent distributions of covariates. Without careful adjustment, this leads to a biased estimate of the causal effect \citep{Jose_stable_weights}.

Mitigation of covariate imbalance in high-dimensional spaces has motivated representation learning schemes for causal inference that seek balance in the learned feature space \citep{Shalit,johansson2016learning}. Despite the empirical success of such methods, it has been recognized that over-enforcing balance can be harmful, as it may inadvertently remove information that is predictive of outcomes \citep{alaa18}. To see this, one may consider an example where a moderately predictive feature might get erased in the learned representation for being highly imbalanced. As such, representation learning-based schemes are sensitive to the hyperparameter that tunes the desired level of imbalance mitigation. 

More classical causal inference approaches seek to match the statistics of the covariates associated with both treatment types \citep{Pearl,Lunceford2004,Rubin_potential_outcomes,holland1986causal}.
Matching methods create a balanced sample by searching for ``similar'' units from the opposite treatment group \citep{Matching}. Matching unfortunately does not scale well to higher dimensions \citep{abadie2006large}, and will often improve balance for some covariates at the expense of balance for others.
Weighting methods assign to each unit a different importance weight so as to match the covariate distributions in different treatment arms after reweighting \citep{Fan_overlap,Lunceford2004}.
In much of the causal inference literature, weighting is employed for {\em average} treatment effect (ATE) estimation over a population. 

In this paper, we employ weighting for {\em conditional} average treatment effect (CATE) estimation. In this context we demonstrate the advantages of learning from regions of good overlap, achieved by employing weighting prior to representation learning. We investigate the coupling of weighting methods \citep{Fan_overlap,Jose_stable_weights,hassanpour2019counterfactual,johansson2018learning} with representation-based causal inference, and demonstrate how the use of properly designed weights alleviates the aforementioned difficulties of representation learning applied to causal inference. We show how targeting an alternative population for empirical loss minimization \citep{Fan_overlap} benefits CATE estimation.
As discussed below, if appropriately designed weights are learned perfectly, then balance is achieved for {\em any} features constituted from the covariates (since balance is achieved in the covariates themselves). However, most weighting methods are computed from the propensity score \citep{Doagostino}, which must be {\em approximated} numerically. Because in practice the weights are always imperfect, exact balance is rarely achieved based on weighting alone, motivating our augmentation of weighting with representation learning. 

This paper makes the following contributions:
($i$) demonstration that the integration of balancing weights alleviates the trade-off between feature balance and predictive power for representation learning;
($ii$) derivation of theoretical results bounding the degree of imbalance as a function of the quality of the propensity model;
$(iii)$ exploration of the benefits of the learned weights and representations as inputs to other learning procedures such as causal forests. We demonstrate that our method, \textit{Balancing Weights Counterfactual Regression (BWCFR)}, mitigates the weaknesses of propensity-weighting and representation learning. In this approach, we do not impose that the features themselves be balanced, as this would likely result in loss of information. Instead, we promote balance for \textit{reweighted} feature distributions, with weights targeting regions for which there is already good overlap.

\section{RELATED WORK}
\paragraph{Representation learning} has been used to achieve balance between treatment group distributions, seeking representations that are both predictive of potential outcomes, and balanced across treatment groups \citep{kallus2018deepmatch,Shalit}. \cite{zhang2020learning} argue that there is often a tradeoff between these objectives, and that over-enforcing balance leads to representations that are less useful for outcome prediction -- our proposal mitigates this tradeoff by enforcing balance between \textit{weighted} feature distributions. Our theory on the discrepancy between the treatment arm distributions (Propositions \ref{prop:KL_bound} and \ref{prop:wass_mmd_bounds}) is also conceptually related to sensitivity modeling in causal analysis \citep{kallus19a}.

\paragraph{Weighting-based methods} typically construct weights as a function of the propensity score to balance covariates \citep{rosenbaum1983a,Lunceford2004}, such as inverse probability weighting (IPW). The performance of these methods critically depends on the quality of the propensity score model and is highly sensitive to the extreme weights \citep{hainmueller2012entropy}. To overcome these limitations, alternative weighting schemes such as Matching Weights \citep{MW}, Truncated IPW \citep{Crump2009} or Overlap Weights \citep{Fan_overlap} seek to change the target population, thereby eliminating extreme weights. Another popular line of solutions directly incorporates covariate balance in constructing the weights \citep{graham2012inverse,diamond2013genetic}, and usually calculate weights via an optimization program with moment matching conditions as the hard \citep{LiFu_matching,hainmueller2012entropy,imai2014covariate} or soft constraints \citep{Jose_stable_weights}. While these bypass propensity score modeling and hence are no longer afflicted by extreme weights, they struggle to scale in high-dimensional settings.

\paragraph{Combining weighting with representation learning} is appealing, as it avoids over-enforcing covariate balance at the expense of predictive power. \cite{hassanpour2019counterfactual} reweight regression terms with inverse probability weights (IPW) estimated from the representations.
Our solution differs in a few ways: First, we do not recommend the use of IPW weights since they often take on extreme values, especially in high dimensions \citep{LiFu_matching}. Second, \cite{hassanpour2019counterfactual} do not state the theoretical benefits of using weights in the first place -- that is, that weights including (but not limited to) the IPW achieve balance between treatment group distributions, given the true propensity. Finally, \cite{hassanpour2019counterfactual} learn the propensity score from the learned representations -- this leads to an optimization procedure where one is required to alternate between learning weights and learning regressors. In contrast, we propose to train a propensity score estimator in the design stage (before any representation learning), then use it to train the regressors to estimate causal effects.

Also related to our setup is the work of \cite{johansson2018learning}, which tackles the slightly different problem of model generalization under design shift, for which they alternately optimize a weighting function and outcome models for prediction. 
Importantly, our work differs from that of \cite{johansson2018learning} in that we learn a propensity score model, and use it to compute the weights, inspired by \cite{crump2008nonparametric,Fan_overlap} -- we argue that this constitutes a more principled approach to learning weights, since we benefit from the so-called \textit{balancing property}, that is: given the true propensity, the reweighted treatment and control arms are guaranteed to be balanced, a desirable property for the estimation of causal effects. The work of \cite{johansson2018learning} does not provide a similar guarantee about the weights allowing achievement of balance, and their learned weights are harder to interpret.

\paragraph{Empowering other causal estimators with the learned balanced representations} is an appealing proposal, motivated by several considerations: ($i$) empirical evidence suggests that there is no ``silver bullet'' causal estimator given the diversity of causal mechanisms investigators might encounter \citep{alaa19a_validating}; ($ii$) many classical solutions (\textit{e.g.}, BART [\citeauthor{chipman2010bart}, \citeyear{chipman2010bart}], causal forests [\citeauthor{wager2018estimation}, \citeyear{wager2018estimation}]) that do not have the luxury of automated representation engineering may possess appealing statistical properties (\textit{e.g.}, built-in CATE uncertainty quantification). Repurposing the learned balanced representations and associated weights can help to free other causal inference procedures from the struggle of resolving the complexity of high-dimensional inputs, thereby boosting both performance and scalability. 

\section{METHODOLOGY}
\subsection{Basic setup}
\paragraph{Assumptions, Identifiability of CATE} Suppose we have $N=N_0+N_1$ units, with $N_{0}$ and $N_{1}$ units in the control and treatment group, respectively. For each unit $i$, we have a binary treatment indicator $T_{i}$ ($T_i=1$ for treated and $T_i=0$ for control), covariates $X_{i}\in\mathcal{X}\subset \mathbb{R}^p$, and two potential outcomes $\{Y_{i}(0),Y_{i}(1)\}\in\mathcal{Y}\subset\mathbb{R}$ corresponding to the control and treatment conditions, respectively. We refer to $Y_i = Y_i(T_i)$ as the factual outcome, and $Y_i^{CF} = Y_i(1-T_i)$ as the counterfactual/unobserved outcome.
The observed dataset is denoted $\mathcal{D}_F = \{X_{i},T_{i},Y_{i}\}_{i=1}^N$. The \textit{propensity score} is $e(x)=\textup{Pr}(T_{i}=1|X_{i}=x)$, and in practice it is estimated from $\{X_i,T_i\}_{i=1}^N$ \citep{rosenbaum1983a}. 

We are interested in predicting the \textit{conditional average treatment effect} (CATE) for a given unit with covariates $x$: $\tau(x)= \mathbb{E}[Y_{i}(1)-Y_{i}(0)|X_i=x].$
As is typical in causal inference, we make the strong ignorability assumptions:
($i$) ignorabililty, which states $\{Y_i(1),Y_i(0)\} \indep T_i \mid X_i$; and ($ii$) positivity, represented as $0<e(x)<1,~~~\forall x \in \mathcal{X}$. Under these assumptions, we can show  that $\tau(x)$ is \textit{identifiable} from observed data \citep{Imbens,Pearl}, and $\tau(x) = \mathbb{E}[Y_{i}|X_{i}=x,T_{i}=1]-\mathbb{E}[Y_{i}|X_{i}=x,T_{i}=0]$.

\paragraph{Target populations} Often causal comparisons are not for a single unit but rather on a \textit{target distribution} of the covariates. Denote $p(x)\triangleq \textup{Pr}(X_i=x)$ as the density of the covariates,
and the densities in the treated and control arms as $p(x|T=1)\triangleq \textup{Pr}(X_i=x|T_i=1)$ and $p(x|T=0)\triangleq \textup{Pr}(X_i=x|T_i=0)$, respectively.
We are interested in performing inference w.r.t. some \textit{target} population density $g(x) \trianglepropto f(x)p(x)$, 
where $f(x)$ is a pre-specified \textit{tilting function} \citep{Fan_overlap}.
Different choices of target densities $g(x)$ give rise to a class of average causal estimands 
\begin{small}
\begin{align}
    \tau_{\textup{ATE},g}\triangleq\mathbb{E}_{g(x)}[\tau(x)]=\int_{\mathcal{X}}\tau(x)g(x)dx,
\end{align}
\end{small}
which includes popular estimands such as the {\it average treatment effect} (ATE) (with $g(x)=p(x)$) and the {\it average treatment effect on the treated} (ATT) (with $g(x)=p(x|T=1)$). 
Table \ref{tab:weights} details popular target populations defined by their tilting functions.
Intuitively, the tilting functions in Table \ref{tab:weights} (with the exception of IPW) place an emphasis on regions of covariate space that are \textit{balanced} in both treatments, \textit{i.e.} regions of overlap, where $e(x)\approx 0.5$ -- this is shown in Figure \ref{fig:tilting_functions}.

\paragraph{Metrics for effect estimation} Suppose we have a model $h(x,t)$ for the expected outcome $\mathbb{E}[Y_{i}|X_{i}=x,T_{i}=t]$ with covariates $x$ under treatment $t$. We can estimate $\tau(x)$ and $\tau_{\textup{ATE},g}$ with
\begin{small}
\begin{align}
    \hat{\tau}(x)&\triangleq h(x,1)-h(x,0),\\\hat{\tau}_{\textup{ATE},g}&\triangleq \mathbb{E}_{g(x)}[\hat{\tau}(x)]\approx \frac{1}{\sum_{i=1}^Nf(X_i)}\sum_{i=1}^{N}f(X_i)\hat{\tau}(X_i).
\end{align}
\end{small}
To evaluate the quality of estimation of the treatment effect on average, we use a metric
$\epsilon_{\textup{ATE},g}\triangleq |\tau_{\textup{ATE},g}-\hat{\tau}_{\textup{ATE},g}|.$ 
To quantify the prediction accuracy of a CATE model $\hat{\tau}$, we use the {\it Precision in Estimation of Heterogeneous Effects} (PEHE)  \citep{Hill} with target density $g(x)$:
\begin{small}
\begin{align} \label{eq:pehe_g}
    \epsilon_{\textup{PEHE},g}&\triangleq \mathbb{E}_{g(x)}[\{\tau(x)-\hat{\tau}(x)\}^2] \\&\approx \frac{1}{\sum_{i=1}^Nf(X_i)}\sum_{i=1}^Nf(X_i)\{\tau(X_i)-\hat{\tau}(X_i)\}^2.
\end{align}
\end{small}
Equation \eqref{eq:pehe_g} is a generalization of the PEHE used in previous work \citep{Shalit,GANITE,louizos2017causal} to target populations $g(x)$.
In the next section, we propose different weighting schemes and discuss how to reweight the units in the treated and control group to match the same target distribution $g(x)$.
\subsection{Balancing weights\label{sec:weight}}
\begin{table}[t]
\caption{\label{tab:weights}
Choices of tilting function $f(x)$ and associated weight schemes $w(x,t)$ in \eqref{eq:balancing_weights}. Note $\mathds{1}(\cdot)$ is the indicator function. We set $\xi=0.1$ as in \cite{Crump2009}.}
\centering
\begin{small}
\begin{tabular}{@{}cl@{}}
\toprule
\textbf{Tilting function $f(x)$} & \textbf{Weight scheme $w(x,t)$}\\
\midrule
 $1$ & Inverse Prob. Weights (IPW)\\
 [2pt]
 ${\mathds{1}(\xi<e(x)<1-\xi)}$ &
Truncated IPW (TruncIPW)  \\
[2pt]
${\text{min}(e(x),1-e(x))}$ & Matching Weights (MW) \\
[2pt]
$e(x)(1-e(x))$ & Overlap Weights (OW)\\
\bottomrule
\end{tabular}
\end{small}
\end{table}
\begin{figure}
    \centering
    \includegraphics[width=\linewidth]{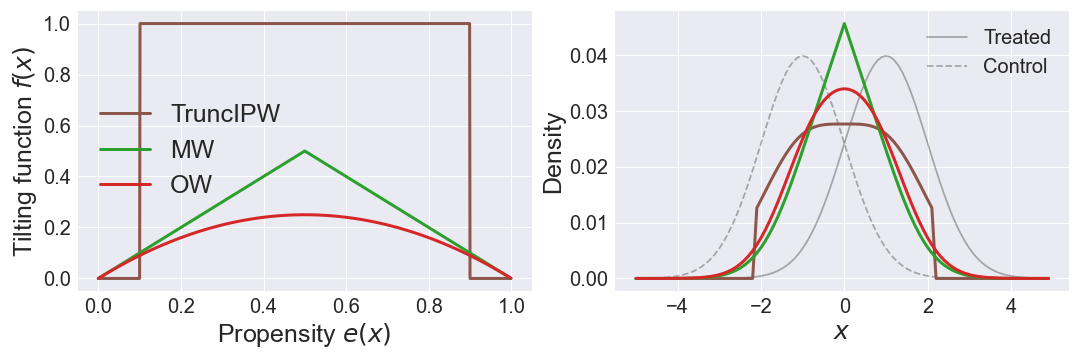}
    \caption{(Left) Tilting functions $f(x)$ used.\\(Right) Illustrative treatment group densities $p(x|T=t)$, and reweighted densities $g(x)\propto f(x)p(x)$ for different $f(x)$, which emphasize regions of good overlap between the treatment and control groups.}
    \label{fig:tilting_functions}
\end{figure}
\paragraph{Balancing with true propensity} For observational studies, typically $p(x|T=1) \neq p(x|T=0)$ due to selection bias resulting from confounding. To achieve balance in the statistics of covariates between the two treatment types, we would like to weight each unit in respective treatment arms towards a common target density $g(x)$. In this study we are particularly interested in a family of target distributions defined by the \textit{balancing weights} \citep{Fan_overlap},
\begin{align}
    w(x,t) = {f(x)}/[{t\cdot e(x) + (1-t)\cdot(1-e(x))}]. \label{eq:balancing_weights}
\end{align}
Table \ref{tab:weights} details popular choices of balancing weights and their corresponding tilting functions. For example, when $f(x)=1$, the weights are the inverse probability weights (IPW) $w(x,1)=1/e(x),w(x,0)=1/(1-e(x))$.
Using balancing weights, we define the reweighted conditional distributions as $g(x|T=1)\trianglepropto w(x,1)p(x|T=1)$ and $g(x|T=0)\trianglepropto w(x,0)p(x|T=0)$. Due to space limitations, all proofs are relegated to the Supplementary Material (SM).
\begin{proposition}[Balancing Property; \citeauthor{Fan_overlap}, \citeyear{Fan_overlap}]
\label{prop:balancing}
Given the true propensity $e(x)$, the reweighted treatment and control arms both equal the target distribution. In other words, $g(x|T=1) = g(x|T=0) = g(x)$.
\end{proposition}

Per Proposition \ref{prop:balancing}, we can balance the treatment and control distributions for estimation of treatment effects prior to any representation learning: the use of balancing weights thus complements the use of representation learning (addressed in Section \ref{sec:rep_learning}) in seeking balance between treatment group distributions -- Figure \ref{fig:tilting_functions} shows the emphasis that balancing weights place on regions with good overlap between treated and control distributions.

\paragraph{Balancing with model propensity} In practice, we do not have access to the true propensity $e(x)$, and we need to estimate it using a model $e_\eta(x)$ with parameters $\eta$ \citep{IPW}.
We plug in the estimated propensity score $e_{\eta}(x)$ in \eqref{eq:balancing_weights} to obtain the approximated balancing weights $w_{\eta}(x,t)$. With the estimated propensity score, Proposition \ref{prop:balancing} no longer holds in general, unless $e_\eta(x) = e(x)$. Given this, we may define the approximate reweighted conditional distributions $g_\eta(x|T=1) \trianglepropto w_\eta(x,1)p(x|T=1)$ and $g_\eta(x|T=0) \trianglepropto w_\eta(x,0)p(x|T=0)$. Though they are not equal in general, we can intuit that, the better the propensity score model, the closer we are to achieving balance between the reweighted treatment arms -- this intuition is supported by Proposition \ref{prop:KL_bound} below.

\begin{assumption}
\label{assum:MSM}
The odds ratio between the model propensity and true propensity is bounded, namely:
\begin{align}
    \exists~\Gamma\geq 1~~s.t.~\forall x\in \mathcal{X},~~\frac{1}{\Gamma} \leq \frac{e(x)(1-e_\eta(x))}{e_\eta(x)(1-e(x))} \leq \Gamma \notag 
\end{align}
\end{assumption}

\begin{proposition}[Generalized Balancing]
\label{prop:KL_bound}
Under Assumption \ref{assum:MSM}, and further assuming that all tilting functions $f$ satisfy $f(x)>0~~\forall x\in\mathcal{X}$, we have:
$$D_{KL}(g_{\eta}(x|T=1)||g_{\eta}(x|T=0)) \leq 2\cdot\log\Gamma,$$
where $D_{KL}$ is the KL-divergence.
\end{proposition}
Proposition \ref{prop:KL_bound} links the (im)balance between reweighted treatment groups to the quality of estimation of the propensity score, quantified by $\Gamma$: the closer $\Gamma$ is to 1, the better the propensity score model. It can be shown immediately that this bound is tight when $\Gamma=1$ -- indeed, perfect estimation of the propensity score yields balance between reweighted treatment and control arms (Proposition \ref{prop:balancing}), so the KL-divergence vanishes.


To estimate treatment effects, we learn a model $h(x,t)$. Such a model is less needed in regions of covariate space that are highly imbalanced (\textit{i.e.}, where $e(x)$ is close to 0 or 1), as for such covariates domain experts generally have a good sense of the appropriate treatment to assign. The MW, OW and TruncIPW weights emphasize regions of covariate space where $e(x)(1-e(x))$ is {\em not} close to zero, and it is this region for which causal predictions are often of most practical utility (the characteristics of $e(x)$ here imply that practitioners are less clear on what the best treatment is). Further, MW, OW and TruncIPW weightings have the advantage of de-emphasizing extreme propensity scores, concentrating on where $e_\eta(x)$ is expected to be most accurate.

MW are a weighting analogue to pair matching \citep{MW}, and OW \citep{Fan_overlap} target the units who are at equipoise (\textit{i.e.}, who are likely to appear in either treatment group). In general, we recommend using OW since there is no cutoff hyperparameter (as in TruncIPW). Further, \cite{Fan_overlap} showed that OW is the minimal asymptotic variance balancing weight for the weighted ATE (though we have yet to show an analogous result for CATE estimation). For a more exhaustive treatment of the different weighting schemes and their interpretation, we refer the readers to \cite{Fan_overlap} and \cite{MW}.

Figure \ref{fig:2D_KDE}(a) illustrates the effect of the Overlap Weights in covariate space -- namely, the emphasis on balanced regions of covariate space.

\subsection{Representation learning with weighting}\label{sec:rep_learning}
Representation learning makes use of an encoder $\Phi:\mathcal{X}\rightarrow \mathcal{R}\subset{\mathbb{R}^{p'}}$ to transform the original covariates to a representation space for CATE prediction using the outcome model $h(\cdot,\cdot):\mathcal{R}\times\{0,1\}\rightarrow \mathcal{Y}$, where $h(\Phi(x),t)$ is the predicted mean potential outcome given covariates $x$ under treatment $t$. The overall model consists of the parameters for $\Phi(x)$ (typically a deep neural network) and the parameters associated with $h(\cdot,t)$, with the latter consisting of two fully-connected neural networks, one for $t=1$ and the other for $t=0$.

Our development is motivated by a generalization bound modified from \cite{Shalit}, which states that under mild technical assumptions the counterfactual prediction error, and consequently, the causal effect prediction error can be upper bounded by a sum of the factual prediction error and a representation discrepancy ({\it i.e.}, quantified imbalance) between the treatment groups. More formally, let
\begin{align}
    \ell_{h,\Phi}(x,t) \triangleq \int_\mathcal{Y}L(y,h(\Phi(x),t))\textup{Pr}(Y(t)=y|X=x)dy,\notag
\end{align}
be the unit loss, where $L(y,y'): \mathcal{Y}\times\mathcal{Y}\rightarrow\mathbb{R}^+$ is a loss function ({\it e.g.}, squared loss $(y-y')^2$). We can further define the expected factual loss w.r.t. the target density under treatment $t \in \{0,1\}$:
\begin{align}
    \epsilon_{F,g}^{T=t} &\triangleq \int_{\mathcal{X}} \ell_{h,\Phi}(x,t)g(x|T=t)dx.
\end{align}
{\it Remark:} This differs from the original setup in \cite{Shalit} in that our expectation is taken w.r.t. the {\em target densities} $g(x|T=t)$ rather than the observational densities $p(x|T=t)$. Under the technical conditions listed in the SM, the following generalization bound holds:
\begin{align}
\label{eq:weighted_balance}
    \epsilon_{\textup{PEHE},g}\leq &2\cdot(\epsilon_{F,g}^{T=1} + \epsilon_{F,g}^{T=0}) + C\\ 
    + &\alpha\cdot \textup{IPM}_G(g_\Phi(r|T=1),g_\Phi(r|T=0))\triangleq B,\notag
\end{align}
where $C$ is a constant w.r.t. model parameters, $r=\Phi(x)$ is the representation for a unit with covariates $x$, and $g_\Phi(r|T=1),g_\Phi(r|T=0)$ are the distributions induced by the map $\Phi$ (which is invertible by assumption) from $g(x|T=1),g(x|T=0)$, respectively. The \textit{integral probability metric} is defined as $\textup{IPM}_G(u,v) \triangleq \underset{m\in G}{\textup{sup}}\int_{\mathcal{R}}m(r)[u(r)-v(r)]dr$ \citep{IPM}, and measures the discrepancy between two distributions $u$ and $v$ by identifying the maximal expected contrast w.r.t. function class $G$. Prominent examples of IPMs include the Wasserstein distance \citep{Wasserstein} and the Maximum Mean Discrepancy (MMD; \citeauthor{GrettonMMD}, \citeyear{GrettonMMD}). With stronger technical assumptions (such as $G$ being the space of all Lipschitz-1 functions, being dense in $L^2$, or derived from a characteristic kernel), the IPM becomes a formal  distance metric for distributions.

Standard decomposition of generalization error typically consists of two parts: the training error and model complexity, where the latter is often formally characterized by measures like Rademacher complexity or VC dimension \citep{shalev2014understanding}. The latter term usually encourages models from a simpler hypothesis space, to avoid overfitting. Compared with the bound in \cite{Shalit}, we can reduce the bound in \eqref{eq:weighted_balance} through proper weighting in the design stage without restricting the representations themselves to be exactly balanced across treatment groups, but rather enforcing that the \textit{reweighted} representations are balanced.
Equivalently, with the proper weighting, we can improve the overall generalization bound by reducing the factual training error without sacrificing the counterfactual generalization. This reconciles the conflict that the IPM and prediction error are at odds.

\begin{figure}
\centering
\begin{subfigure}[b]{0.5\textwidth}
    \includegraphics[width=\linewidth]{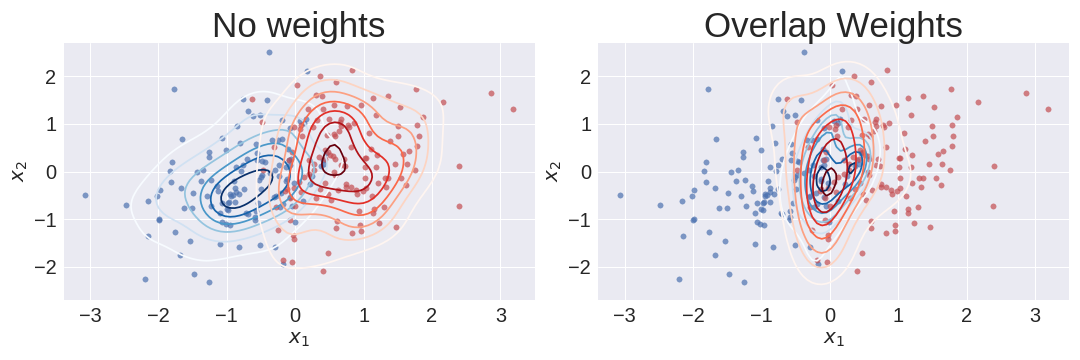}
  \caption{2D treatment (red) \& control (blue) covariate samples, with distributions $g_\eta(x|T=1)$ (red) and $g_\eta(x|T=0)$ (blue), shown as kernel density estimates (contour plots -- left: unweighted, right: overlap-weighted).}
  \label{fig:row1_KDE} 
\end{subfigure}

\begin{subfigure}[b]{0.5\textwidth}
    \includegraphics[width=\linewidth]{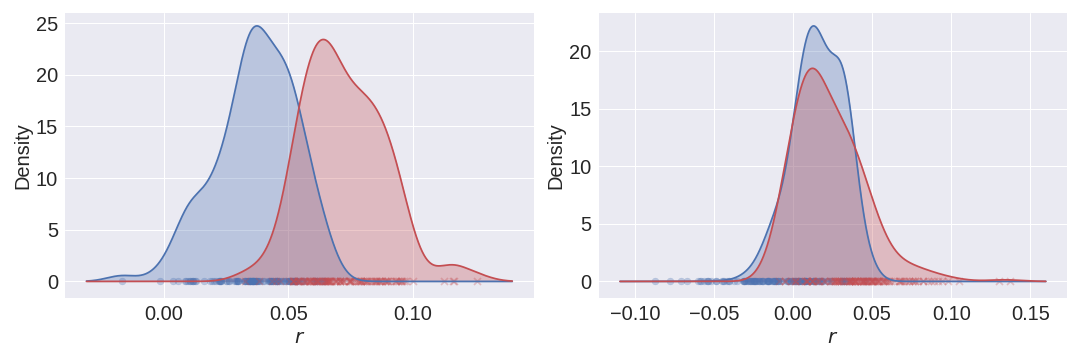}
   \caption{Learned representation distributions $g_{\Phi,\eta}(r|T=1)$ (red) and $g_{\Phi,\eta}(r|T=0)$ (blue).}
  \label{fig:row2_KDE}
\end{subfigure}
\caption[]{Illustrative example with highly imbalanced treatment arms. The columns are the weight schemes used for training the outcome models and weighting the representations. (a) shows that the overlap weights (OW) focus the learning on regions of overlap in covariate space. (b) illustrates that the weighting schemes can help achieve balance in representation space under severe selection bias.}
\label{fig:2D_KDE}
\vspace{-4mm}
\end{figure}
If the weights $w_\eta(x,t)$ are computed perfectly (\textit{i.e.}, if the propensity-score model satisfies $e_\eta(x) = e(x), \forall x\in \mathcal{X}$), the IPM term in \eqref{eq:weighted_balance} vanishes -- a direct consequence of Proposition \ref{prop:balancing}. However, as we do not know the true propensity in practice, we approximate the bound as
\begin{align}
\label{eq:bound_approx}
    B \approx &2\cdot(\epsilon_{F,g_\eta}^{T=1}+ \epsilon_{F,g_\eta}^{T=0}) + C \\+ &\alpha\cdot \textup{IPM}_G(g_{\Phi,\eta}(r|T=1),g_{\Phi,\eta}(r|T=0)) \notag ,
\end{align}
where $g_{\Phi,\eta}(r|T=1),g_{\Phi,\eta}(r|T=0)$ are the distributions induced by the map $\Phi$ from the reweighted distributions $g_\eta(x|T=1)$ and $g_\eta(x|T=0)$. In practice, we use the Wasserstein distance and the MMD as the IPM in equation \eqref{eq:bound_approx}.
\vspace{2mm}
\begin{proposition}\label{prop:wass_mmd_bounds}
Under Assumption \ref{assum:MSM}, assuming the representation space $\mathcal{R}$ is bounded, and assuming the tilting functions satisfy $f(x)>0~\forall x \in \mathcal{X}$, the following bounds hold:
\begin{small}
\begin{align}
    &\mathcal{W}(g_{\Phi,\eta}(r|T=1),g_{\Phi,\eta}(r|T=0)) \leq \textup{diam}(\mathcal{R})\sqrt{\log\Gamma}; \notag\\
    &\textup{MMD}_k(g_{\Phi,\eta}(r|T=1),g_{\Phi,\eta}(r|T=0)) \leq 2\sqrt{C_k\log\Gamma},\notag
\end{align}
\end{small}
where $\mathcal{W}$ is the Wasserstein distance, $\textup{diam}(\mathcal{R}) \triangleq \textup{sup}_{r,r'\in\mathcal{R}}||r-r'||_2$, $\textup{MMD}_k$ is the MMD with kernel $k$, and $C_k \triangleq \textup{sup}_{r\in\mathcal{R}}k(r,r)$.
\end{proposition}
Proposition \ref{prop:wass_mmd_bounds} bounds the IPM by the factor $\Gamma$ which quantifies the quality of the propensity score model as in Assumption \ref{assum:MSM} -- it is again easy to show that the bounds are tight when $\Gamma=1$. This result is intuitive, and shows that, the better the propensity model, the more balanced the reweighted feature distributions.\\
Here the IPM term may be seen as a {\em correction} to the weights, addressing errors manifested by imperfections in the estimated propensity score. However, since much of the balance is achieved by the weights, it is less likely that the weighted IPM term will remove predictive features. Figure \ref{fig:2D_KDE}(b) illustrates how weighting can achieve balance in representation space. 
The weighted density plots show that the learned weighted representations become more balanced compared with the unweighted one. Weighting achieves a similar effect as the IPM term in balancing the representations, but it does not enforce that the (unweighted) empirical distributions of the representations need to be matched across treatments.
\subsection{Implementation}
We train a propensity score model $e_\eta(x)$ by minimizing $\mathcal{L}_{prop}(\eta)$ w.r.t. $\eta$, where:
\begin{small}
\begin{align}
    \mathcal{L}_{prop}(\eta) = -\sum_{i=1}^N&\{\frac{T_i}{N_1}\cdot\log[\sigma(s_\eta(X_i))] \\&+ \frac{1-T_i}{N_0}\cdot\log[1-\sigma(s_\eta(X_i))]\}.\notag
\end{align}
\end{small}
$\sigma(z)\triangleq 1/[1+\exp(-z)]$, and $s_\eta(x)$ is a fully-connected neural network with $e_\eta(x) \triangleq \sigma(s_\eta(x))$. 
Once $e_\eta(x)$ is trained, we learn the parameters of the encoder $\Phi(x)$ and the outcome models $h(\Phi(x),1)$ and $h(\Phi(x),0)$. We can show that the approximation in \eqref{eq:bound_approx} leads to the following finite-sample objective, which we minimize w.r.t. $h,\Phi$:
\begin{align}
\label{eq:objective}
\mathcal{L}(h,\Phi) \triangleq &\mathcal{L}_F(h,\Phi)\\ + &\alpha\cdot \text{IPM}_G\left({\hat{g}_{\Phi,\eta}}(r|T=1),{\hat{g}_{\Phi,\eta}}(r|T=0)\right)\notag
\end{align}
where $\mathcal{L}_F(h,\Phi)$ is a Monte Carlo approximation of $\epsilon_{F,g_\eta}^{T=1}+\epsilon_{F,g_\eta}^{T=0}$ and $\hat{g}_{\Phi,\eta}(r|T=t)$ is the empirical approximation of $g_{\Phi,\eta}(r|T=t)$ ($t \in \{0,1\}$), defined as:
\begin{small}
\begin{align}
\mathcal{L}_F(h,\Phi) &\triangleq \frac{1}{N}\sum_{i=1}^N w_{\eta}(X_i,T_i) \, \left(Y_i - h(\Phi(X_i),T_i)\right)^2 \notag,\\ 
{\hat{g}_{\Phi,\eta}}(r|T=t) &\triangleq \sum_{i: T_i=t} \frac{w_{\eta}(X_i,t)}{\sum_{j: T_j=t}w_{\eta}(X_j,t)}\delta(r-\Phi(X_i)) \notag.
\end{align}
\end{small}
$\delta(r-z)$ is a point mass centered at $z$ and $\Phi(\cdot), h(\cdot,1), h(\cdot,0)$ are fully-connected neural networks. 
More details on how to obtain the finite-sample approximation in \eqref{eq:objective} and how to compute the weighted IPM term in practice are provided in the SM.
\section{EXPERIMENTS}
\subsection{Synthetic data}
\paragraph{Data generating process} We wish to understand the effect of distribution imbalance (the extent to which the treatment and control distributions differ) on the performance of our methods for CATE estimation. Specifically, we construct datasets for which we vary the distribution imbalance and the amount of confounding. Consider the following data-generating mechanism:
    \begin{itemize}[leftmargin=*]
    \item Fix $\sigma_X,\sigma_Y, \rho, \theta \in \mathbb{R}$. Set $\beta_0,\beta_\tau,\gamma \in \mathbb{R}^p$ to be $p^*$-sparse vectors ($i.e.$, $||\beta_0||_0 = ||\beta_\tau||_0 = ||\gamma||_0 = p^*$), and further set $supp(\beta_0) = supp(\beta_\tau) \triangleq \mathcal{B}$, $\mathcal{G} \triangleq supp(\gamma)$, and the \textit{confounding parameter} $\Omega \triangleq |\mathcal{B}\cap\mathcal{G}|$.
    \item For simplicity, set $\gamma = \Tilde{\gamma}\cdot\mathds{1}_\mathcal{G}$, where $\mathds{1}_\mathcal{G}\in \{0,1\}^p$ is a binary vector with ones at elements of $\mathcal{G}$, and $\Tilde{\gamma}\geq 0$ is the \textit{imbalance parameter}. Note $||\gamma||_2=\Tilde{\gamma}\cdot p^*$.
    \item Draw $X_i,T_i, Y_i(1),Y_i(0)$ as follows:
        \begin{align}
        &X_i \sim \mathcal{MVN}(0,\sigma_X^2 [(1-\rho)I_p + \rho 1_p1_p^T]),\\
        &T_i|X_i \sim \textup{Bernoulli}(\sigma(X_i^T\gamma)),\\
        &\epsilon_i \sim \mathcal{N}(0,\sigma_Y^2),\hspace{0.1in}Y_i(0) = X_i^T\beta_0 + \epsilon_i,\\&Y_i(1) = X_i^T\beta_0 + X_i^T\beta_\tau + \theta + \epsilon_i.
        \end{align}
\end{itemize}
This data-generating process satisfies the assumptions of ignorability and overlap. We construct multiple such datasets by varying the distribution imbalance and amount of confounding, as follows:
\begin{figure}
\centering
\includegraphics[width=\linewidth]{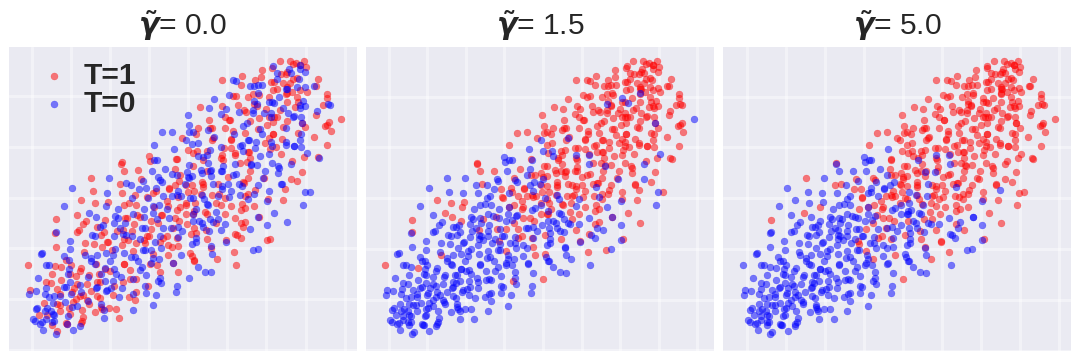}
\caption{\label{fig:gamma_overlap_clean}
Higher $\Tilde{\gamma}$ leads to larger imbalance between treatment groups.}
\end{figure}

\textit{Distribution imbalance:} We increase the distribution imbalance by increasing $\Tilde{\gamma}$ in the range [0,5]. Figure \ref{fig:gamma_overlap_clean} shows a $t$-SNE plot of the 50-dimensional covariates, for different values of $\Tilde{\gamma}$.\\
\textit{Level of confounding:} We increase the level of confounding by increasing $\Omega$, \textit{i.e.}, the extent to which the same covariates are predictive of treatment and potential outcomes. We vary $\Omega$ to be equal to $p^*$ (``high confounding''), $\tfrac{p^*}{2}$ (``moderate confounding''), and 0 (``low confounding''). In the ``low confounding'' setting, ($\Omega=0$), there is still some confounding by way of the correlation $\rho$ between the covariates.

In total we generate 33 datasets (3 values of $\Omega~~\times$ 11 values of $||\gamma||_2$). For more details on the data-generating process, see the SM.

\paragraph{Results} We compare the weighted-model performance across the 33 datasets generated as discussed above, and we compare against using no weights. For a fair comparison, we fix all hyperparameters with the exception of the IPM regularization strength $\alpha$ (for details on hyperparameters, see the SM). Figure \ref{fig:min_alpha_clean} shows the performance of each method for all datasets. For a given dataset and weight scheme, we select the $\alpha$ that minimizes $\epsilon_{\textup{PEHE},p}$. We picked the $\alpha$ minimizing the true $\epsilon_{\textup{PEHE},p}$ (which includes knowledge of counterfactual outcomes) to avoid introducing any noise in the comparisons via a proxy such as a 1-nearest-neighbor imputation (1NNI) of missing potential outcomes. 
For the remaining experiments on real data (Sections \ref{sec:IHDP} and \ref{sec:CF}), we use 1NNI, which makes no use of counterfactual outcomes, for model selection in order to compare with existing work.

From Figure \ref{fig:min_alpha_clean}, one can immediately see the benefit of using a weighted objective (weighted regression + weighted IPM) over its unweighted counterpart. More specifically, the MW, OW, and TruncIPW weights do well in comparison with the other weight schemes, especially in settings of high imbalance (\textit{i.e.}, high values of $\Tilde{\gamma}$). On the other hand, IPW is numerically unstable \citep{LiFu_matching} and yields only marginally better results than its unweighted counterpart, so we do not recommend its use as a weighting scheme.
This provides empirical evidence for the fact that weighted CATE models, though trained to perform well on a \textit{target} population $g(x)$ (namely, for the non-IPW weights, regions of good overlap), vastly improve CATE estimation on the \textit{observed} population $p(x)$.
We also compared the performance of our models on the target populations (\textit{i.e.}, as measured by $\sqrt{\epsilon_{\textup{PEHE},g}}$), and found that the weight schemes perform well on the respective populations they target. For details about the performance on target populations, see the SM.
\begin{figure}
    \centering
    \includegraphics[width=\linewidth]{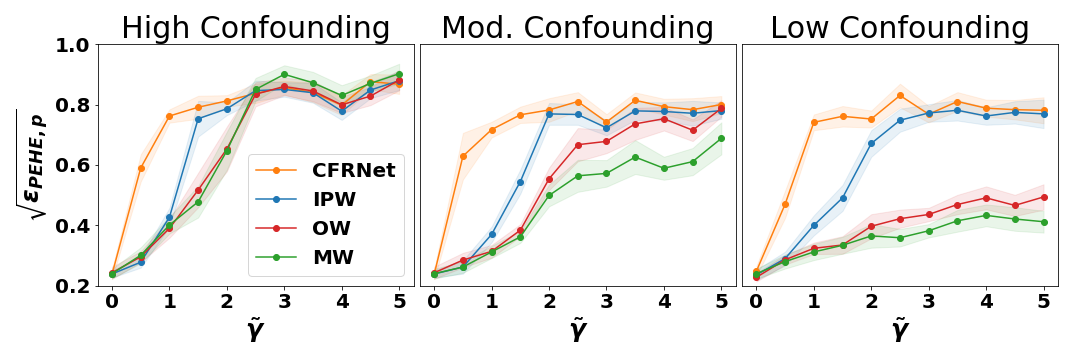}
    \caption{$\sqrt{\epsilon_{\textup{PEHE},p}}$ \textit{vs.} dataset imbalance parameter $\Tilde{\gamma}$ for different confounding settings (high, moderate, low). The colored bands are standard errors over 20 realizations. TruncIPW was omitted to avoid clutter, because it was similar to OW and MW. CFRNet is from \cite{Shalit},
    which uses propensity-independent weights $(1-T_i)/N_0+T_i/N_1$. Lower is better.}
    \label{fig:min_alpha_clean}
\end{figure}

\paragraph{Benefit of weighted IPM regularization}
\begin{figure}
    \centering
    \includegraphics[width=\linewidth]{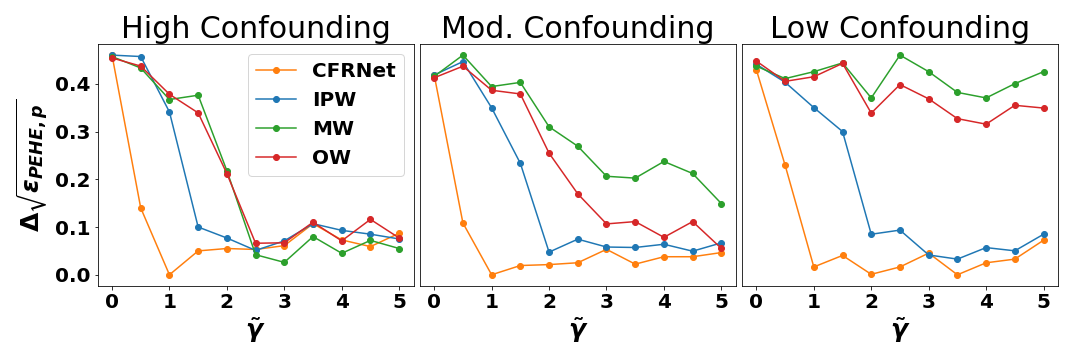}
    \caption{Improvement in $\sqrt{\epsilon_{\textup{PEHE},p}}$ between the model with and without the IPM term, denoted $\Delta\sqrt{\epsilon_{\textup{PEHE},p}}$, \textit{vs.} imbalance parameter $\Tilde{\gamma}$. The colored bands are standard errors over 20 realizations. Higher is better.}
    \label{fig:IPM_ablation_clean}
\end{figure}
We seek to understand the benefit of the weighted IPM term in our objective formulation. We make this comparison by taking the difference between the best $\sqrt{\epsilon_{\textup{PEHE},p}}$ across all values of $\alpha$ (\textit{i.e.}, the plots shown in Figure \ref{fig:min_alpha_clean}), and the $\sqrt{\epsilon_{\textup{PEHE},p}}$ for $\alpha=0$ (\textit{i.e.}, without the IPM term).
The benefit of using the weighted IPM term vs. an unweighted IPM term is immediately visible from Figure \ref{fig:IPM_ablation_clean}, especially in cases of moderate to high imbalance. A likely explanation for this is that the IPM term is attempting to match the \textit{weighted} distributions in representation space rather than the unweighted ones, which means it is less prone to ``erasing'' information from confounders.
Finally, the IPM term when $\Tilde{\gamma}$=0 should theoretically be 0, since the treatment and control covariate distributions are the same. However, there still may be imbalance between the drawn \textit{samples}, but the Wasserstein distance would vanish with increasing sample size \citep{sriperumbudur2009integral}.
\subsection{Infant Health and Development Program} \label{sec:IHDP}
The Infant Health and Development Program (IHDP) dataset \citep{Hill} is semi-simulated (real covariates with simulated outcomes) measuring the effect of home visits from a trained provider on children's cognitive test scores. This dataset has a more realistic covariate distribution than the above synthetic data, but we cannot control the degree of imbalance.
We report out-of-sample results on the IHDP1000 dataset  from \cite{Shalit} in Table \ref{tab:IHDP_results}, showing competitive performance both in terms of CATE prediction ($\sqrt{\epsilon_{\textup{PEHE},p}}$) and ATE prediction ($\epsilon_{\textup{ATE},p}$). For details on model selection and training, see the SM. For results with an IPW-based solution, we point the readers. We note from Table $\ref{tab:IHDP_results}$ that our method (BWCFR) outperforms many classical causal inference methods, such as the causal forest. This is in part because our method benefits from automated representation learning (via the map $\Phi(x)$) upstream of the outcome models $h(\Phi(x),1),h(\Phi(x),0)$. In the next section, we explore what happens when we leverage the learned features and weights to benefit classical methods with appealing statistical features (\textit{e.g.}, CATE uncertainty estimates).
\begin{savenotes}
\begin{table}[!h]
    \caption{Results on IHDP1000 test set. The top block consists of baselines from recent work. The bottom block consists of our proposed methods. Lower is better.}
    \centering
    \begin{scriptsize}
    \begin{tabular}{@{}lcc@{}}
    \toprule
     Model &  $\sqrt{\epsilon_{\textup{PEHE},p}}$ &   $\epsilon_{\textup{ATE},p}$\\
     \midrule
     OLS-1 \citep{johansson2016learning}& $5.8 \pm .3$ & $.94 \pm .06$\\
        OLS-2  \citep{johansson2016learning}& $2.5 \pm .1$ & $.31 \pm .02$\\ 
        BLR   \citep{johansson2016learning}& $5.8 \pm .3$ & $.93 \pm .05$\\
        $k$-NN \citep{crump2008nonparametric}& $4.1 \pm .2$ & $.79 \pm .05$\\
        BART  \citep{chipman2010bart}& $2.3 \pm .1$ & $.34 \pm .02$\\
        Random Forest  \citep{breiman2001random}& $6.6 \pm .3$ & $.96 \pm .06$\\
        Causal Forest \citep{wager2018estimation}& $3.8 \pm .2$ & $.40 \pm .03$ \\
        BNN \citep{johansson2016learning} & $2.1 \pm .1$ & $.42 \pm .03$\\
        TARNET \citep{Shalit}& $.88\pm .02 $ & $.26\pm .01 $\\
        
        CFRNet-Wass \citep{Shalit}& $.76 \pm .02$ & $.27 \pm .01$\\
        CFR-ISW\footnote{These methods reported their results on the IHDP100 dataset (equivalent to IHDP1000, but with 100 repetitions instead of 1000). For a comparison on IHDP100, see SM.\label{fn:ihdp100}} \citep{hassanpour2019counterfactual} & $.70 \pm .1$ & $.19 \pm .03$ \\
        RCFR\footnoteref{fn:ihdp100}\citep{johansson2018learning} & $.67 \pm .05$ & -  \\
        CMGP \citep{alaa2017bayesian}& $.74 \pm .11$ & -\\

        DKLITE \citep{zhang2020learning}& $.65 \pm .03$ & - \\
        [5pt]
        BWCFR-MW (Ours) &  .66 $\pm$ .02 &  {\bf .18 $\pm$ .01} \\
        BWCFR-OW (Ours) &   .65 $\pm$ .02 &  {\bf .18 $\pm$ .01} \\
        BWCFR-TruncIPW (Ours) &  {\bf .63 $\pm$ .01} &  .19 $\pm$ .01 \\
        \bottomrule
    \end{tabular}
    \end{scriptsize}
    \label{tab:IHDP_results}
\end{table}
\end{savenotes}
\subsection{Improving causal forest with the balanced representations learned}
\label{sec:CF}
We further examine the extent to which the learned balanced representations of our proposal can facilitate other causal learning algorithms. In particular, we quantitatively assess the potential gains for causal forests (CF; \citeauthor{wager2018estimation}, \citeyear{wager2018estimation}), and report our findings in Table \ref{tab:causal_forest_ihdp_acic}. In the first experiment we evaluate the performance difference with and without the learned balanced representation and weights on the IHDP100 dataset \citep{Hill,Shalit} w.r.t. both the individual and population level metrics ({\it i.e.}, $\sqrt{\epsilon_{\textup{PEHE},p}}$,  $\epsilon_{\textup{ATE},p}$). Also, we examine the proportion of datasets (out of 77) in the ACIC2016 benchmark \citep{dorie2019,alaa19a_validating} for which the learned representations and weights improve ({\it i.e.}, $\%\downarrow\sqrt{\epsilon_{\textup{PEHE},p}}$ and $\%\downarrow\epsilon_{\textup{ATE},p}$, respectively), compared to a ``vanilla'' CF model. For both datasets we observe substantial gains in both CATE and ATE estimation (relative to the vanilla CF trained on the original covariates), which demonstrates the effectiveness of using pre-balanced representations and weights, in this case learned by our model, to augment other causal models.

To note, our methods on IHDP (bottom-right of Table \ref{tab:IHDP_results}) still outperform that of the causal forest (in terms of $\sqrt{\epsilon_{\textup{PEHE},p}}$), even when the causal forest has access to the same balanced representations and weights learned (bottom-left of Table \ref{tab:causal_forest_ihdp_acic}). One potential explanation is that tree-based learner lacks the sophistication to decode the rich representation encoded by a more flexible neural net. 
It would be interesting to explore an end-to-end optimization strategy that combines our proposed representation engineering and the causal forest model; we leave this for future work.
For details about hyperparameter tuning and additional analyses, see the SM.

\begin{table}[!h]
    \centering
    \caption{Causal forest (CF) results. The top block is a vanilla CF model. The bottom block consists of causal forest models using the learned representations and weights. The bottom block rows are the weights used in the objective \eqref{eq:objective} and as the per-sample weights to train the CF. The left block shows $\sqrt{\epsilon_{\textup{PEHE},p}}$ and $\epsilon_{\textup{ATE},p}$ results on the IHDP dataset (lower is better), and the right block shows $\%\downarrow\sqrt{\epsilon_{\textup{PEHE},p}}$ and $\%\downarrow\epsilon_{\textup{ATE},p}$ results on the ACIC2016 dataset (higher is better).}
    \begin{scriptsize}
    \begin{tabular}{@{}lcc@{\hskip 8pt}cc@{}}
    \toprule
        &\multicolumn{2}{c}{IHDP100}&\multicolumn{2}{c}{ACIC2016}\\
     &  $\sqrt{\epsilon_{\textup{PEHE},p}}$ &  $\epsilon_{\textup{ATE},p}$ & $\downarrow\sqrt{\epsilon_{\textup{PEHE},p}}$& $\downarrow\epsilon_{\textup{ATE},p}$
     \\
     \midrule
     CF & 3.54 $\pm$ .58 & .47 $\pm$ .06 & - & -\\
     [5pt]
     CF + MW & 1.51 $\pm$ .31 & .20 $\pm$ .03 & 92.2\% & 89.6\% \\
     CF + OW & 1.59 $\pm$ .31 & .19 $\pm$ .03 & 93.5\% & 85.7\%\\
     CF + TruncIPW & 1.55 $\pm$ .35 & .22 $\pm$ .03 & 87.0\% & 71.4\%\\   \bottomrule
    \end{tabular}
    \end{scriptsize}
    \label{tab:causal_forest_ihdp_acic}
\end{table}

\section{CONCLUSIONS}
We show that the use of balancing weights complements representation learning in mitigating covariate imbalance.
Our claims are supported with theoretical results and evaluations on synthetic datasets and realistic test benchmarks, reporting better or competitive performance throughout. Further, we demonstrated how our learned balanced features can augment other causal inference procedures, towards the goal of building more reliable and accurate hybrid solutions. Directions for future work include \textit{learning} the tilting function $f$ rather than selecting it in order to determine an ``optimal'' target population, as well as exploring more advanced weighting approaches \citep{hainmueller2012entropy,Jose_stable_weights,ozery2018adversarial}. Additionally, the current form of Assumption \ref{assum:MSM} is somewhat restrictive (it posits the existence of $\Gamma$ such that odds ratio is bounded \textit{for all} $x$). Softening it (\textit{e.g.}, by bounding the odds ratio, integrated over the population of interest) would provide a more realistic form of Assumption \ref{assum:MSM}, which would make Proposition \ref{prop:KL_bound} more useful. Finally, an end-to-end approach to training the propensity and regression models (as in \citeauthor{hassanpour2020}, \citeyear{hassanpour2020}) seems like a promising alternative to our current two-step training procedure.
\subsubsection*{Acknowledgements}
The research reported here was supported in part by the DOE, NSF and ONR.
\setcounter{section}{0}
\renewcommand{\thesection}{S\arabic{section}}
\setcounter{table}{0}
\renewcommand{\thetable}{S\arabic{table}}
\setcounter{figure}{0}   
\renewcommand{\thefigure}{S\arabic{figure}}
\newpage
\onecolumn
\aistatstitleSM{Supplementary Material: Counterfactual Representation Learning with Balancing Weights}
\section{Theory}
\subsection{Proof of Proposition 1}
\begin{proposition_SM}[Balancing Property]
\label{prop:balancing_SM}
Given the true propensity score $e(x)$, the reweighted treatment and control arms both equal the target distribution. In other words, $g(x|T=1) = g(x|T=0) = g(x)$
\end{proposition_SM}

\begin{proof}
\begin{align}
g(x|T=1)\trianglepropto w(x,1)p(x|T=1) = \frac{f(x)}{e(x)}p(x|T=1) = \frac{f(x)\textup{Pr}(X=x|T=1)}{\textup{Pr}(T=1|X=x)} \propto f(x)p(x) \propto g(x)
\end{align}
Similarly, we can also show that $g(x|T=0) = g(x)$.
\end{proof}

\subsection{Proof of Proposition 2}
\begin{assumption_SM}
\label{assum:MSM_SM}
The odds ratio between the model propensity score and true propensity score is bounded, namely:
\begin{align}
    \exists~\Gamma\geq 1~~s.t.~\forall x\in \mathcal{X},~~\frac{1}{\Gamma} \leq \frac{e(x)(1-e_\eta(x))}{e_\eta(x)(1-e(x))} \leq \Gamma 
\end{align}
\end{assumption_SM}
This assumption is conceptually related to the Marginal Sensitivity Model of \cite{kallus19a} in that it measures the gap between two propensity functions -- we use it here to quantify the gap between true and model propensities rather than the degree of unobserved confounding.
\begin{proposition_SM}[Generalized Balancing]
\label{prop:KL_bound_SM}
Under Assumption \ref{assum:MSM_SM}, and assuming that all tilting functions $f$ satisfy $f(x)>0~~\forall x\in\mathcal{X}$, we have:
$$D_{KL}(g_{\eta}(x|T=1)||g_{\eta}(x|T=0)) \leq 2\cdot\log\Gamma,$$
where $D_{KL}$ is the KL-divergence.
\end{proposition_SM}
\begin{proof}
First, we write the reweighted treatment group distribution as follows:
\begin{align}
    g_\eta(x|T=1) \trianglepropto w_\eta(x,1)p(x|T=1) =  \frac{f_\eta(x)}{e_\eta(x)}p(x|T=1),
\end{align} where we write $f_\eta$ since the tilting function is (in general) computed from the propensity score model. With $f(x)$ the ``true'' tilting function (\textit{i.e.}, the tilting function computed from the true propensity $e(x)$), we may write:
\begin{align}
    g_\eta(x|T=1)\propto \frac{f(x)}{e(x)}p(x|T=1)\frac{f_\eta(x)}{f(x)}\frac{e(x)}{e_\eta(x)} \propto g(x) \frac{f_\eta(x)}{f(x)}\frac{e(x)}{e_\eta(x)}
\end{align}
where the last equality holds from Proposition 1.
Similarly, we can write the reweighted control group distribution as $$ g_\eta(x|T=0)\propto g(x) \frac{f_\eta(x)}{f(x)}\frac{1-e(x)}{1-e_\eta(x)}.$$
Now, computing the KL-divergence between $g_\eta(x|T=1)$ and $g_\eta(x|T=0)$, we get:
\begin{align}
D_{KL}(g_\eta(x|T=1)||g_\eta(x|T=0)) &= \int_\mathcal{X} g_\eta(x|T=1)\log\left[\frac{\frac{1}{Z_1} g(x)\frac{f_\eta(x)}{f(x)}\frac{e(x)}{e_\eta(x)}}{\frac{1}{Z_0}g(x)\frac{f_\eta(x)}{f(x)}\frac{1-e(x)}{1-e_\eta(x)}}\right]dx \label{eq:proof_KLbound}
\end{align}
where $Z_1 \triangleq \int_\mathcal{X}g(x)\frac{f_\eta(x)}{f(x)}\frac{e(x)}{e_\eta(x)} dx$ and $Z_0 \triangleq \int_\mathcal{X} g(x)\frac{f_\eta(x)}{f(x)}\frac{1-e(x)}{1-e_\eta(x)} dx$.
Simplifying \eqref{eq:proof_KLbound} further, we get:
\begin{align}
D_{KL}(g_\eta(x|T=1)||g_\eta(x|T=0)) &= \int_\mathcal{X} g_\eta(x|T=1)\left[ \log\frac{Z_0}{Z_1} + \log\frac{e(x)(1-e_\eta(x))}{e_\eta(x)(1-e(x))}\right]dx\\
&\overset{(*)}{\leq} \int_\mathcal{X} g_\eta(x|T=1)\left[ \log\frac{Z_0}{Z_1} + \log\Gamma\right]dx = \log\frac{Z_0}{Z_1} + \log\Gamma \label{eq:proof_KLbound2}
\end{align}
where $(*)$ holds from Assumption \ref{assum:MSM_SM}.
Notice that we may relate $Z_1$ and $Z_0$ as follows:
\begin{align}
    Z_0 \triangleq\int_\mathcal{X}g(x)\frac{f_\eta(x)}{f(x)}\frac{1-e(x)}{1-e_\eta(x)}dx  \overset{(**)}{\leq} \int_\mathcal{X}g(x)\frac{f_\eta(x)}{f(x)}\Gamma\frac{e(x)}{e_\eta(x)} dx = \Gamma Z_1
\end{align}
where $(**)$ also holds from Assumption \ref{assum:MSMSM}.
Hence $\log\frac{Z_0}{Z_1} \leq \log\Gamma$ -- plugging this into \eqref{eq:proof_KLbound2} yields:
\begin{align}
    D_{KL}(g_\eta(x|T=1)||g_\eta(x|T=0)) \leq \log\Gamma + \log\Gamma = 2\log\Gamma.
\end{align}
\end{proof}
\begin{corollary} \label{cor:KL_bound}
The bound presented in Proposition \ref{prop:KL_bound_SM} also holds for the induced distributions $g_{\Phi,\eta}(r|T=1),g_{\Phi,\eta}(r|T=0)$ from $g_\eta(x|T=1),g_\eta(x|T=0)$ (respectively) via any invertible map $\Phi:\mathcal{X}\rightarrow\mathcal{R}$ (with inverse $\Psi$), namely:
\begin{align}
    D_{KL}(g_{\Phi,\eta}(r|T=1)||g_{\Phi,\eta}(r|T=0)) \leq 2\log\Gamma.
\end{align}
\end{corollary}
\begin{proof}
To see this, we can write:
\begin{align}
g_{\Phi,\eta}(r|T=1) &\overset{(*)}{\propto} g_\eta(\Psi(r)|T=1)|\textup{det}(\Psi')|\propto \frac{f(\Psi(r))}{e(\Psi(r))}p(\Psi(r)|T=1)\frac{f_\eta(\Psi(r))}{f(\Psi(r))}\frac{e(\Psi(r)}{e_\eta(\Psi(r))}|\textup{det}(\Psi')| \\&\propto g(\Psi(r))\frac{f_\eta(\Psi(r))}{f(\Psi(r))}\frac{e(\Psi(r))}{e_\eta(\Psi(r))}|\textup{det}(\Psi')| \label{eq:g_phi_1}
\end{align}
where $\textup{det}(\Psi')$ is the determinant of the Jacobian of $\Psi$, and $(*)$ holds from the change-of-variables formula. Similarly, we can write $g_{\Phi,\eta}(r|T=0)$ as:
\begin{align} \label{eq:g_phi_0}
    g_{\Phi,\eta}(r|T=0) \propto g(\Psi(r))\frac{f_\eta(\Psi(r))}{f(\Psi(r))}\frac{1-e(\Psi(r))}{1-e_\eta(\Psi(r))}|\textup{det}(\Psi')|
\end{align}
Computing the KL divergence between \eqref{eq:g_phi_1} and \eqref{eq:g_phi_0} is then similar to the proof of Proposition \ref{prop:KL_bound_SM}, and the same bound holds.
\end{proof}
\subsection{Proof of Proposition 3}
\begin{definition}
The total variation distance (TVD) between distributions $p$ and $q$ on $\mathcal{R}$ is defined as 
\begin{align}
    \delta(p,q) \triangleq \frac12\cdot \underset{m: ||m||_\infty \leq 1}{\textup{sup}} \Big\{\int_\mathcal{R}m(r)(p(r)-q(r))dr\Big\}
\end{align}
\end{definition}
\begin{lemma}\label{lem:tvd_bound}
Under Assumption \ref{assum:MSM_SM}, the total variation distance between the reweighted representation distribution for the treatment and control groups is upper bounded as:
\begin{align}
    \delta(g_{\Phi,\eta}(r|T=1),g_{\Phi,\eta}(r|T=0)) \leq \sqrt{\log\Gamma}
\end{align}
\end{lemma}
\begin{proof}
\begin{align}
    \delta(g_{\Phi,\eta}(r|T=1),g_{\Phi,\eta}(r|T=0)) \overset{(*)}{\leq}
    \sqrt{\frac12 D_{KL}(g_{\Phi,\eta}(r|T=1)||g_{\Phi,\eta}(r|T=0))} 
   \overset{(**)}{\leq} \sqrt{\log\Gamma}
\end{align}
where $(*)$ follows from Pinsker's inequality, and $(**)$ follows from Corollary \ref{cor:KL_bound}.
\end{proof}

\begin{proposition_SM}\label{prop:wass_mmd_bounds_SM}
Under Assumption \ref{assum:MSM_SM}, assuming the representation space $\mathcal{R}$ is bounded, and assuming the tilting functions satisfy $f(x)>0~\forall x \in \mathcal{X}$, the following bounds hold:
\begin{align}
    &\mathcal{W}(g_{\Phi,\eta}(r|T=1),g_{\Phi,\eta}(r|T=0)) \leq \textup{diam}(\mathcal{R})\sqrt{\log\Gamma} \notag\\
    &\textup{MMD}_k(g_{\Phi,\eta}(r|T=1),g_{\Phi,\eta}(r|T=0)) \leq 2\sqrt{C_k\log\Gamma},
\end{align}
where $\mathcal{W}$ is the Wasserstein distance, $\textup{diam}(\mathcal{R}) \triangleq \textup{sup}_{r,r'\in\mathcal{R}}||r-r'||_2$, $\textup{MMD}_k$ is the MMD with kernel $k$, and $C_k \triangleq \textup{sup}_{r\in\mathcal{R}}k(r,r)$.
\end{proposition_SM}
\begin{proof}
\begin{align}
\mathcal{W}(g_{\Phi,\eta}(r|T=1),g_{\Phi,\eta}(r|T=0)) \overset{(*)}{\leq} \textup{diam}(\mathcal{R})\delta(g_{\Phi,\eta}(r|T=1),g_{\Phi,\eta}(r|T=0)) \overset{(**)}{\leq} diam(\mathcal{R})\sqrt{\log\Gamma}
\end{align}
where $(*)$ holds from Theorem 4 of \cite{TVD_Wass_bound}, and $(**)$ holds from Lemma \ref{lem:tvd_bound}.
\begin{align}
    \textup{MMD}_k(g_{\Phi,\eta}(r|T=1),g_{\Phi,\eta}(r|T=0)) \overset{(*)}{\leq} 2\sqrt{C_k}\delta(g_{\Phi,\eta}(r|T=1),g_{\Phi,\eta}(r|T=0))\overset{(**)}{\leq} 2\sqrt{C_k\log\Gamma}
\end{align}
where $(*)$ holds from Theorem 14-ii of \cite{sriperumbudur2009integral}, and $(**)$ holds from Lemma \ref{lem:tvd_bound}.
\end{proof}

\subsection{Relationship between $\epsilon_{\textup{PEHE},p}$ and $\epsilon_{\textup{PEHE},g}$}\label{sec:pehe_equivalence_supp}
In this section, we establish a relationship between $\epsilon_{\textup{PEHE},p}$ and $\epsilon_{\textup{PEHE},g}$ which explains why targeting the population $g(x)\trianglepropto f(x)p(x)$ for ITE prediction may also aid ITE prediction on the original covariate distribution $p(x)$.
As a reminder, Table \ref{tab:weights_supp} and Figure \ref{fig:target_populations} shows the different tilting functions of interest and their corresponding weighting schemes. $e(x)\triangleq \textup{Pr}(T=1|X=x)$ is the propensity score.
The weight schemes we use here have been carefully examined in classical causal inference literature \citep{Crump2009,Fan_overlap,MW}. Specifically, the Matching Weights \citep{MW} were designed as a weighting analogue to matching, the Truncated IPW weights \citep{Crump2009} were used to estimate a low-variance average treatment effect for a subpopulation, and the Overlap Weights \citep{Fan_overlap} were proven to minimize (out of all the possible balancing weights) the asymptotic variance of the estimated weighted average treatment effect. Figure \ref{fig:target_populations} shows how TruncIPW, MW, and OW place a specific emphasis on regions of good overlap in covariate space.
\begin{table}[!ht]
\caption{\label{tab:weights_supp} \small Choices of tilting function $f(x)$ and associated weight schemes $w(x,t)$ (see equation (6) in the main text). Note $\mathds{1}(\cdot)$ is the indicator function. We set $\xi=0.1$ as in \cite{Crump2009}.}
\centering
\begin{tabular}{@{}cc@{}}
\toprule
\textbf{Tilting function $f(x)$} & \textbf{Associated weight scheme $w(x,t)$}\\
\midrule
 $1$ & Inverse Probability Weights (IPW)\\ {\footnotesize${\mathds{1}(\xi<e(x)<1-\xi)}$}&
Truncated IPW (TruncIPW) \\
${\text{min}(e(x),1-e(x))}$ & Matching Weights (MW) \\
$e(x)(1-e(x))$ & Overlap Weights (OW)\\
\bottomrule
\end{tabular}
\end{table}
\begin{figure}[H]
    \centering
    \includegraphics[width=\textwidth]{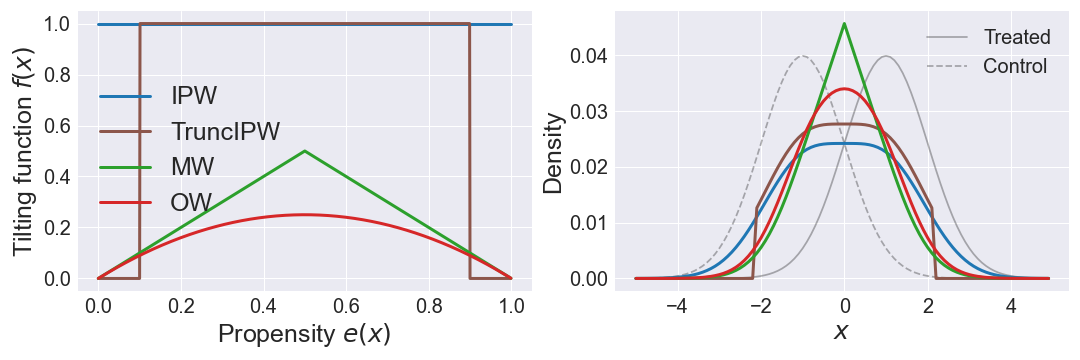}
     \caption{(Left) Tilting functions $f(x)$ used. (Right) Illustrative treatment group densities $p(x|T=t)$, and reweighted densities $g(x)\propto f(x)p(x)$ for different $f(x)$. TruncIPW, MW, and OW specifically emphasize regions of good overlap between the treatment and control groups.}
    \label{fig:target_populations}
\end{figure}
\begin{definition}[$\delta$-strict overlap]
$\exists\delta\in(0,0.5): \forall x\in\mathcal{X}~\delta<e(x)<1-\delta$. 
\end{definition}
\begin{definition}[$\epsilon_{\textup{PEHE},g}$]
 $$\epsilon_{\textup{PEHE},g}(\hat{\tau}) \triangleq \int_{\mathcal{X}} (\tau(x)-\hat{\tau}(x))^2g(x)dx, $$
 where $\tau(x)\triangleq \mathbb{E}[Y(1)-Y(0)|X=x]$ is the true individual treatment effect, and $\hat{\tau}$ is an estimate of $\tau(x)$. We often omit $\hat{\tau}$ from $\epsilon_{\textup{PEHE}}(\hat{\tau})$ for brevity.
\end{definition}
\begin{proposition_SM}\label{prop:pehe_equiv}
Assuming $\delta$-strict overlap, for all the tilting functions presented in Table \ref{tab:weights_supp} (for $f(x)=\mathds{1}(\xi<e(x)<1-\xi)$, the additional condition $\delta\geq\xi$ is required), we have:
$$
    A_f \cdot \epsilon_{PEHE,g}(\hat{\tau})\leq \epsilon_{PEHE,p}(\hat{\tau})\leq B_f\cdot\epsilon_{PEHE,g}(\hat{\tau}),
$$
where $A_f$ and $B_f$ are constants depending on $f$, $p(x)\triangleq \textup{Pr}(X_i=x)$, and $g(x)\trianglepropto f(x)p(x)$
\end{proposition_SM}
\begin{proof}
For all the tilting functions $f(x)$ in Table 1, we have $\textup{sup}_x f(x)<\infty$. Assuming $\delta$-strict overlap (and for $f(x)=\mathds{1}(\xi<e(x)<1-\xi)$, assuming $\delta\geq \xi$) we also get (for all $f(x)$ in Table \ref{tab:weights_supp}) that $\textup{inf}_x f(x)>0$.
Since $g(x)\trianglepropto f(x)p(x)$ by definition (with $p(x)\triangleq \textup{Pr}(X=x)$ the marginal covariate density), we may write:
\begin{align}
    \epsilon_{PEHE,g} &\triangleq \int_{\mathcal{X}} (\tau(x)-\hat{\tau}(x))^2g(x)dx\\
    &= \int_{\mathcal{X}} (\tau(x)-\hat{\tau}(x))^2\frac{f(x)p(x)}{Z_f}dx
\end{align}
where $Z_f \triangleq \int_\mathcal{X} f(x)p(x)dx$.
We may bound this expression above and below via:

\begin{align}
   &\int_{\mathcal{X}} (\tau(x)-\hat{\tau}(x))^2\frac{\textup{inf}_x[f(x)]p(x)}{Z_f}dx \leq \epsilon_{PEHE,g} \leq \int_{\mathcal{X}} (\tau(x)-\hat{\tau}(x))^2\frac{\textup{sup}_x[f(x)]p(x)}{Z_f}dx\\
   \Rightarrow &\frac{\textup{inf}_x[f(x)]}{Z_f}\int_{\mathcal{X}} (\tau(x)-\hat{\tau}(x))^2p(x)dx \leq \epsilon_{PEHE,g} \leq \frac{\textup{sup}_x[f(x)]}{Z_f}\int_{\mathcal{X}} (\tau(x)-\hat{\tau}(x))^2p(x)dx\\
   \Rightarrow &\frac{\textup{inf}_x[f(x)]}{Z_f}\cdot \epsilon_{PEHE,p}\leq \epsilon_{PEHE,g} \leq \frac{\textup{sup}_x[f(x)]}{Z_f}\cdot \epsilon_{PEHE,p}
\end{align}

Defining $B_f \triangleq \frac{Z_f}{\textup{inf}_x[f(x)]}$ and $A_f \triangleq \frac{Z_f}{\textup{sup}_x[f(x)]}$:
\begin{align}
    \frac{1}{B_f} \cdot \epsilon_{PEHE,p}\leq \epsilon_{PEHE,g} \leq \frac{1}{A_f}\cdot\epsilon_{PEHE,p}
\end{align}
Which we may also write as:
\begin{align}
    A_f \cdot \epsilon_{PEHE,g}\leq \epsilon_{PEHE,p} \leq B_f \cdot\epsilon_{PEHE,g} \label{eq:pehe_norms}
\end{align}
\end{proof}

Proposition \ref{prop:pehe_equiv} gives a ``two birds, one stone'' property, whereby $\epsilon_{\textup{PEHE},p}$ may also be minimized when $\epsilon_{\textup{PEHE},g}$ is minimized. This is a possible justification for why targeting the population $g(x)$ (via minimizing an upper bound on $\epsilon_{\textup{PEHE},g}$) may also benefit ITE estimation on the observed population $p(x)$.

\subsection{$\epsilon_{\textup{PEHE},g}$ bound}\label{sec:bound_pehe_supp}
Here, we establish conditions for which the bound on $\epsilon_{\textup{PEHE},g}$ (equation (8) in the main text) holds.
\begin{proposition_SM}
Assuming the encoder $\Phi$ is invertible, and assuming $\frac{1}{\alpha}\ell_{h,\Phi}\in G$ for a function class $G$ and a constant $\alpha$, we have:\\
\begin{align}
    \epsilon_{\textup{PEHE},g}\leq 2\cdot(\epsilon_{F,g}^{T=1} + \epsilon_{F,g}^{T=0}) + \alpha\cdot \textup{IPM}_G(g_\Phi(r|T=1),g_\Phi(r|T=0)) + C\triangleq B,
\end{align}
where $C$ is a constant w.r.t. model parameters, and $g_\Phi(r|T=t)$ is the distribution induced by the invertible map $\Phi$ from the distribution $g(x|T=t)$ (for $t\in\{0,1\}$).
\end{proposition_SM}
\begin{proof}
The proof follows straightforwardly by applying Theorem 1 from \cite{Shalit} on the target population $g(x)$.
\end{proof}

\section{Finite-sample objective}\label{sec:finite_sample}
From equation (8) in the main text, we know $\epsilon_{\textup{PEHE},g}\leq B$.
From equation (9) in the main text, we have:
\begin{align}
    B \approx 2\cdot(\epsilon_{F,g_\eta}^{T=1}+ \epsilon_{F,g_\eta}^{T=0}) + \alpha\cdot \textup{IPM}_G(g_{\Phi,\eta}(r|T=1),g_{\Phi,\eta}(r|T=0)) + C\label{eq:bound_approx_sup}
\end{align}

We would like to obtain a finite-sample estimate of $B$ (shown in equation (11) in the main text).
\subsection{Finite-sample factual error terms $\epsilon_{F,g_\eta}^{T=1}$,$\epsilon_{F,g_\eta}^{T=0}$}
We will start by estimating the first 2 terms in \eqref{eq:bound_approx_sup}, choosing $\epsilon_{F,g_\eta}^{T=1}$ WLOG.

\begin{align}
\epsilon_{F,g_\eta}^{T=1} &\triangleq \int_{\mathcal{X}}\ell_{h,\Phi}(x,1)g_\eta(x|T=1)dx = \int_{\mathcal{X}}\ell_{h,\Phi}(x,1)\frac{w_\eta(x,1)p(x|T=1)}{Z_1}dx
\end{align}
where $Z_1 = \int_\mathcal{X} w_\eta(x,1)p(x|T=1)dx$, and $\ell_{h,\Phi}(x,t) \triangleq \int_\mathcal{Y}L(y,h(\Phi(x),t))\textup{Pr}(Y(t)=y|X=x)dy$, with $L(y,y') = (y-y')^2$.

We may approximate $\epsilon_{F,g_\eta}^{T=1}$ as:
\begin{align}
    \epsilon_{F,g_\eta}^{T=1} &\approx \frac{1}{Z_1\cdot n_1}\sum_{i\in\mathcal{B}:T_i=1}w_\eta(X_i,1)(Y_i-h(\Phi(X_i),1))^2 \label{eq:err_1_approx}
\end{align}
where $\mathcal{B}$ is a sampled batch, and $n_1 \triangleq \sum_{i\in \mathcal{B}} T_i$.\\
The target distribution $g(x)$ is defined as $g(x)\triangleq \frac{f(x)p(x)}{Z}$ where $Z\triangleq \int_\mathcal{X} f(x)p(x)$.
We make the following approximation for $Z_1$:
\begin{align}
Z_1 &\triangleq \int_{\mathcal{X}}w_\eta(x,1)p(x|T=1) \approx \int_{\mathcal{X}}w(x,1)p(x|T=1)dx = \int_{\mathcal{X}}\frac{f(x)}{e(x)}p(x|T=1)dx \\
&= \int_{\mathcal{X}}\frac{f(x)}{\textup{Pr}(T=1)}p(x)dx = \frac{Z}{\textup{Pr}(T=1)}\int_{\mathcal{X}}\frac{f(x)p(x)}{Z}dx = \frac{Z}{\textup{Pr}(T=1)}\int_{\mathcal{X}}g(x)dx \\
&= \frac{Z}{\textup{Pr}(T=1)} \approx \frac{Z\cdot N}{N_1}\end{align}
where $N$ is the number of samples in the dataset, and $N_1 = \sum_{i=1}^N T_i$ is the number of treatment samples in the dataset.
We explicitly construct the batches $\mathcal{B}$ of size $n$ such that $n_1/n = N_1/N$, so we get:
\begin{align}
Z_1 \approx \frac{Z\cdot n}{n_1}
\end{align}
Finally, we plug in the above approximation of $Z_1$ into \eqref{eq:err_1_approx} to get:
\begin{align}
\epsilon_{F,g_\eta}^{T=1} &\approx \frac{1}{Z\cdot n}\sum_{i\in\mathcal{B}:T_i=1}w_\eta(X_i,1)(Y_i-h(\Phi(X_i),1))^2 \label{eq:err_1_approx_final}
\end{align}
Similarly, we may approximate $ \epsilon_{F,g_\eta}^{T=0}$ as:\\
\begin{align}
 \epsilon_{F,g_\eta}^{T=0} &\approx
    \frac{1}{Z\cdot n}\sum_{i\in\mathcal{B}:T_i=0}w_\eta(X_i,0)(Y_i-h(\Phi(X_i),0))^2 \label{eq:err_0_approx_final}
\end{align}
We also tried the approximations $Z_1 \approx \frac{1}{n_1}\sum_{i\in\mathcal{B}:T_i=1}w_\eta(X_i,1)$ and $Z_1 \approx \frac{1}{N_1}\sum_{i:T_i=1}w_\eta(X_i,1)$ (and similar approximations for $Z_0$), but they did not work well in practice.
\subsection{Finite-sample IPM term}
Finally, we seek a Monte-Carlo approximation of the third term in \eqref{eq:bound_approx_sup}. Recalling the definition of $ g_\eta(x|T=1)$, we have:
\begin{align}
    g_\eta(x|T=1) &\triangleq \frac{w_\eta(x,1)p(x|T=1)}{Z_1} \label{eq:def_g_eta_sup}
\end{align}
where $Z_1 \triangleq \int_\mathcal{X}w_\eta(x,1)p(x|T=1)dx$.

We assume that $\Phi(\cdot): \mathcal{X}\rightarrow \mathcal{R}$ is an invertible transformation (with inverse $\Psi$), so it induces distributions $g_{\Phi,\eta}(r|T=1)$ and $p_{\Phi}(r|T=1)$ from $g_\eta(x|T=1)$ and $p(x|T=1)$, respectively.
From the change of variables formula:
\begin{align}
    g_{\Phi,\eta}(r|T=1) = g_\eta(\Psi(r)|T=1)\cdot |\textup{det}(\Psi')|
\end{align}
where $\Psi'$ is the Jacobian of $\Psi$, and $\textup{det}(\cdot)$ is the determinant. From \eqref{eq:def_g_eta_sup}, we get:\\
\begin{align}
    g_{\Phi,\eta}(r|T=1) = \frac{w_\eta(\Psi(r),1)}{Z_1}\cdot p(\Psi(r)|T=1)\cdot|\textup{det}(\Psi')|
\end{align}
By the change of variables formula on the last 2 terms above, we get:
\begin{align}
     g_{\Phi,\eta}(r|T=1) = \frac{w_\eta(\Psi(r),1)}{Z_1}\cdot p_\Phi(r|T=1) \label{eq:ratio_1}
\end{align}
We may approximate $g_{\Phi,\eta}(r|T=1)$ from samples in a batch $\mathcal{B}$ as:
\begin{align}
    &g_{\Phi,\eta}(r|T=1) \approx \frac{1}{\sum_{i\in \mathcal{B}: T_i=1}w_\eta(X_i,1)/Z_1}\sum_{i\in \mathcal{B}: T_i=1} \frac{w_\eta(X_i,1)}{Z_1}\delta(r-\Phi(X_i))\\
    &= \frac{1}{\sum_{i\in \mathcal{B}: T_i=1}w_\eta(X_i,1)}\sum_{i\in \mathcal{B}: T_i=1} w_\eta(X_i,1)\delta(r-\Phi(X_i)) \triangleq \hat{g}_{\Phi,\eta}(r|T=1) \label{eq:ratio_1_approx}
\end{align}
Where $\delta(r-z)$ is a point-mass centered at $z$.
Similarly, we can approximate $g_{\Phi,\eta}(r|T=0)$ as:
\begin{align}
    g_{\Phi,\eta}(r|T=0) \approx \hat{g}_{\Phi,\eta}(r|T=0) \triangleq  \frac{1}{\sum_{i\in \mathcal{B}: T_i=0}w_\eta(X_i,0)}\sum_{i\in \mathcal{B}: T_i=0} w_\eta(X_i,0)\delta(r-\Phi(X_i)) \label{eq:ratio_0_approx}
\end{align}

\subsection{Putting it all together}
Plugging \eqref{eq:err_1_approx_final}, \eqref{eq:err_0_approx_final}, \eqref{eq:ratio_1_approx}, and \eqref{eq:ratio_0_approx} into \eqref{eq:bound_approx_sup}, we may write an approximation of the bound $B$ (from \eqref{eq:bound_approx_sup}) as:
\small
\begin{align}
\frac{2}{Z\cdot n}\sum_{i\in \mathcal{B}}w_\eta(X_i,T_i)(Y_i-h(\Phi(X_i),T_i))^2
    +\alpha\cdot \textup{IPM}_G(\hat{g}_{\Phi,\eta}(r|T=1),\hat{g}_{\Phi,\eta}(r|T=0))
    + C
\end{align}
\normalsize
The above has the same argmin as:\\
\small
\begin{align}
    \mathcal{L}(h,\Phi,\mathcal{B}) &=  \frac{1}{n}\sum_{i\in \mathcal{B}}w_\eta(X_i,T_i)(Y_i-h(\Phi(X_i),T_i))^2
    +\alpha'\cdot \textup{IPM}_G(\hat{g}_{\Phi,\eta}(r|T=1),\hat{g}_{\Phi,\eta}(r|T=0))
\end{align}
\normalsize
for some constant $\alpha'$ (which we leave as $\alpha$ in the main text to avoid introducing more notation).\\
This is the finite-sample objective presented in equation (11) of the main text -- the version presented here is over a mini-batch $\mathcal{B}$, but we omitted this detail from the main text for simplicity.
\subsection{Weighted Integral Probability Metric (IPM) computation}
As a reminder, IPMs \citep{IPM} are defined as follows:
\begin{align}
    \textup{IPM}_G(u,v) = \underset{m\in G}{\textup{sup}}\int_\mathcal{R} m(r)[u(r)-v(r)]dr
\end{align}
where $G$ is a function class, and $u$ and $v$ are probability measures.
In our implementation, similar to \cite{Shalit}, we use two kinds of IPMs: namely, the Wasserstein distance, by setting the function class $G = \{m: ||m||_L \leq 1 \}$ to be the set of 1-Lipschitz functins, and the Maximum Mean Discrepancy (MMD; \citeauthor{GrettonMMD}, \citeyear{GrettonMMD}), by setting $G = \{m: ||m||_\mathcal{H} = 1\}$ to be the set of norm-1 functions in a reproducing kernel Hilbert space $\mathcal{H}$. In this section, we provide details for how to compute these IPMs between the reweighted distributions $g_{\Phi,\eta}(r|T=1)$ and $g_{\Phi,\eta}(r|T=0)$, which is the last term in our objective in equation (11) of the main text.

\paragraph{Finite-sample weighted MMD}
First, suppose the class of functions $G = \{m: ||m||_\mathcal{H}=1\}$ is the set of norm-1 functions in a reproducing kernel Hilbert space (RKHS) $\mathcal{H}$ with corresponding kernel $k(\cdot,\cdot)$. $\textup{IPM}_G$ is then equivalent to the Maximum-Mean Discrepancy (MMD).
From Lemma 4 in \cite{GrettonMMD}, the squared MMD is equal to:
\begin{align}
\textup{MMD}^2(p,q) &= ||\mu_p-\mu_q||_\mathcal{H}^2\\
&= \langle \mu_p,\mu_p\rangle_{\mathcal{H}} + \langle \mu_q,\mu_q\rangle_{\mathcal{H}} - 2\cdot\langle \mu_p,\mu_q\rangle_{\mathcal{H}}
\end{align}
where $\mu_p(\cdot) \triangleq \mathbb{E}_{x \sim p}[k(\cdot,x)]$ and $\mu_q$ is defined similarly.\\
We now wish to get a finite sample estimate of $\textup{MMD}^2(g_{\Phi,\eta}(r|T=1),g_{\Phi,\eta}(r|T=0))$. Assuming $\Phi$ is invertible with inverse $\Psi$, from equation \eqref{eq:ratio_1}, we have:
\begin{align}
g_{\Phi,\eta}(r|T=1) = \frac{w_\eta(\Psi(r),1)p_\Phi(r|T=1)}{Z_1}\\
g_{\Phi,\eta}(r|T=0) = \frac{w_\eta(\Psi(r),0)p_\Phi(r|T=0)}{Z_0}
\end{align}

Where $Z_t = \int_{\mathcal{R}}w_\eta(\Psi(r),t)p_\Phi(r|T=t)$ for $t\in \{0,1\}$.\\
WLOG, we now seek a finite-sample estimate of $\langle \mu_1,\mu_1\rangle_\mathcal{H}$, where $\mu_1 \triangleq \mathbb{E}_{R \sim g_{\Phi,\eta}(r|T=1)}[k(\cdot,R)]$.\\
\begin{align}
\mu_1(\cdot) \approx &\sum_{i\in\mathcal{B}:T_i=1}\frac{w_\eta(X_i,1)}{\sum_{i\in\mathcal{B}:T_i=1}w_\eta(X_i,1)}k(\cdot,\Phi(X_i))\\
\Rightarrow \langle\mu_1,\mu_1\rangle_{\mathcal{H}} \approx &\frac{\sum_{i\in\mathcal{B}:T_i=1}\sum_{j\in\mathcal{B}:T_j=1}w_\eta(X_i,1)w_\eta(X_j,1)k(\Phi(X_i),\Phi(X_j))}{[\sum_{i\in\mathcal{B}:T_i=1}w_\eta(X_i,1)]^2}
\end{align}
Using the V-statistic version \citep{GrettonMMD} of the above, we get:
\begin{align}
 \langle\mu_1,\mu_1\rangle_{\mathcal{H}} \approx &\frac{\sum_{i\in\mathcal{B}:T_i=1}\sum_{j\in\mathcal{B}:T_j=1,j\neq i}w_\eta(X_i,1)w_\eta(X_j,1)k(\Phi(X_i),\Phi(X_j))}{\sum_{i\in\mathcal{B}:T_i=1}\sum_{j\in\mathcal{B}:T_j=1,j\neq i }w_\eta(X_i,1)w_\eta(X_j,1)} \label{eq:mmd_treatment}
\end{align}
Similarly, we can approximate $\langle\mu_0,\mu_0\rangle_\mathcal{H}$ as:
\begin{align}
\langle\mu_0,\mu_0\rangle_{\mathcal{H}} \approx &\frac{\sum_{i\in\mathcal{B}:T_i=0}\sum_{j\in\mathcal{B}:T_j=0,j\neq i}w_\eta(X_i,0)w_\eta(X_j,0)k(\Phi(X_i),\Phi(X_j))}{\sum_{i\in\mathcal{B}:T_i=0}\sum_{j\in\mathcal{B}:T_j=0,j\neq i }w_\eta(X_i,0)w_\eta(X_j,0)]} \label{eq:mmd_control}
\end{align}
Finally, we similarly approximate $\langle\mu_1,\mu_0\rangle_\mathcal{H}$ as:
\begin{align}
\langle\mu_1,\mu_0\rangle_{\mathcal{H}} \approx &\frac{\sum_{i\in\mathcal{B}:T_i=1}\sum_{j\in\mathcal{B}:T_j=0}w_\eta(X_i,1)w_\eta(X_j,0)k(\Phi(X_i),\Phi(X_j))}{\sum_{i\in\mathcal{B}:T_i=1}\sum_{j\in\mathcal{B}:T_j=0}w_\eta(X_i,1)w_\eta(X_j,0)]} \label{eq:mmd_cross}
\end{align}

Finally, we get the finite-sample estimate of $\textup{MMD}^2(g_\eta(r|T=1),g_\eta(r|T=0))$ via:\\
\begin{align}
\textup{MMD}^2(g_\eta(r|T=1),g_\eta(r|T=0)) \approx
\eqref{eq:mmd_treatment}+ \eqref{eq:mmd_control} -2\cdot \eqref{eq:mmd_cross}
\end{align}

In practice we set $k(\cdot,\cdot)$ to either be a linear kernel, i.e. $k(R_i,R_j) = R_i^TR_j$, or a RBF kernel, i.e. $k(R_i,R_j) = \exp(-\frac{||R_i-R_j||_2^2}{\sigma^2})$, where $\sigma$ is set to 0.1.

\paragraph{Finite-sample weighted Wasserstein distance}
For the finite sample approximation of the weighted Wasserstein distance, we use Algorithm 3 of \cite{Cuturi} (shown here in Algorithm \ref{algo:sk} for convenience), with the entropic regularization strength set to $\lambda=10$, and vectors $a\in\mathbb{R}^{n_1}$, $b\in\mathbb{R}^{n_0}$ and matrix $M\in\mathbb{R}^{n_1\times n_0}$ set to:
\begin{align}
    a^{(i)}= \frac{w_\eta(X_i,1)}{\sum_{k\in\mathcal{B}:T_k=1}w_\eta(X_k,1)};\hspace{0.15in}
    b^{(j)}= \frac{w_\eta(X_j,0)}{\sum_{k\in\mathcal{B}:T_k=0}w_\eta(X_k,0)};\hspace{0.15in}
    M^{(i,j)}= ||\Phi(X_i)-\Phi(X_j)||_2
\end{align}
We fix the number of Sinkhorn iterations to $S=10$.

\begin{algorithm}[!h]
        \begin{algorithmic}
	\caption{Sinkhorn-Knopp Algorithm for weighted Wasserstein distance approximation}\label{algo:sk}
          \STATE \textbf{Input} batch $\mathcal{B}$, entropic regularization parameter $\lambda\in\mathbb{R}$, number of Sinkhorn iterations $S$, encoder $\Phi(\cdot)$, propensity score parameters $\eta$\\
          \hfill\\
          \STATE $n_1 = \sum_{i\in\mathcal{B}}T_i$;~~~ $n_0 = \sum_{i\in\mathcal{B}}(1-T_i);$\\
          \hfill\\
          \STATE Compute weight vectors $a\in\mathbb{R}^{n_1},b\in\mathbb{R}^{n_0}$ of empirical approximations $\hat{g}_{\Phi,\eta}(r|T=1),\hat{g}_{\Phi,\eta}(r|T=0)$, as:
          \STATE    $a^{(i)}= \frac{w_\eta(X_i,1)}{\sum_{k\in\mathcal{B}:T_k=1}w_\eta(X_k,1)} ~~\forall i\in \mathcal{B}: T_i=1$;\hspace{0.15in} $b^{(j)}= \frac{w_\eta(X_j,0)}{\sum_{k\in\mathcal{B}:T_k=0}w_\eta(X_k,0)}~~\forall j\in \mathcal{B}: T_j=0;$\\
          \hfill\\
          \STATE Compute pairwise distance matrix $M\in \mathbb{R}^{n_1\times n_0}$ between treatment \& control representations, as:
          \STATE $M^{(i,j)}= ||\Phi(X_i)-\Phi(X_j)||_2 ~~\forall i\in \mathcal{B}: T_i=1, \forall j\in \mathcal{B}: T_j=0$;\\
          \hfill\\
		  \STATE $K=\exp(-\lambda M)$; \hspace{0.15in} \% elementwise exponential
		  \STATE $\tilde{K}=\textup{diag}(a^{-1}) K$;
		  \STATE Initialize $u=a$;

		  \FOR{$s\in[0,...,S-1]$}
		  \STATE \texttt{$u=1./(\tilde{K}(b./(K^Tu)))$}; \hspace{0.15in} \% Sinkhorn iterations
		  \ENDFOR
		  \STATE $v=b./(K^Tu).$
		  \STATE  $T^{\star}_\lambda=\textup{diag}(u)K\textup{diag}(v)$;\\
		  \hfill\\
		  \STATE \textbf{return} $ \textup{Wass}(\hat{g}_{\Phi,\eta}(r|T=1),\hat{g}_{\Phi,\eta}(r|T=0)) \approx \sum_{i,j} T^{\star (i,j)}_\lambda M^{(i,j)}$
        \end{algorithmic}
     \end{algorithm}
\section{Experimental details}
\subsection{Toy experiment}\label{sec:toy_exp_details}
\paragraph{Data-generating parameters}
We specify $\beta_0,\beta_\tau$,and $\gamma \in \mathbb{R}^p$ (from Section 4.1 in the main text) as follows:
\begin{align}
\beta_0 \triangleq \Tilde{\beta_0}\cdot\mathds{1}_\mathcal{B};\hspace{0.5in}\beta_\tau \triangleq\Tilde{\beta_\tau}\cdot\mathds{1}_\mathcal{B};\hspace{0.5in}\gamma \triangleq\Tilde{\gamma}\cdot\mathds{1}_\mathcal{G}
\end{align}
where $\Tilde{\beta_0},\Tilde{\beta_\tau},\Tilde{\gamma} \in \mathbb{R}$,~~$\mathcal{B} \triangleq supp(\beta_0) = supp(\beta_\tau)$, and $\mathcal{G} \triangleq supp(\gamma)$.

Note that $\Tilde{\beta_0},\Tilde{\beta_\tau},\Tilde{\gamma}$ can be used to control the magnitudes of $\beta_0,\beta_\tau,\gamma$ respectively.
Table \ref{tab:toy_details} indicates the value of every parameter used to generate the toy dataset described in the previous section. With the values in Table \ref{tab:toy_details}, we get 33 datasets (each simulated 20 times), indexed by $\Tilde{\gamma}$ (which controls the imbalance) and $\Omega \triangleq |\mathcal{B} \cap \mathcal{G}|$ (which controls the level of confounding).

\begin{table}[h!]
    \centering
    \caption{Data-generating parameters for Section 4.1 in the main text}
\begin{tabular}{p{0.8in}p{3in}p{1.4in}}
\toprule
     \textbf{Parameter} & \textbf{Description} & \textbf{Value/Range}\\
     \midrule
     $N$ & number of data points & 525/225/250 (train/val/test)\\
     $p$ & dimension of covariates &50\\
     $p^*$ & non-zero dimensions in $\beta_0,\beta_\tau,\gamma$ & 20\\
     $\sigma_X^2$ & variance of covariates & 0.05\\
     $\sigma_Y$ & variance of additive Gaussian noise in potential outcomes $Y_i(0),Y_i(1)$ & 1.0\\
     $\rho$ & correlation between covariates & 0.3 \\
     $\Tilde{\beta_0}$ & effective magnitude of $\beta_0$& 1.0\\
     $\Tilde{\beta_\tau}$ & effective magnitude of $\beta_\tau$& 0.3\\
     $ \Tilde{\gamma}$ & imbalance parameter & $\{0,0.5,1.00,...,5.00\}$\\
     $\mathcal{B}$ & support of $\beta_0,\beta_\tau$ & \\
     $\mathcal{G}$ & support of $\gamma$ & \\
     $\Omega$ & confounding parameter: $|\mathcal{B}\cap \mathcal{G}|$ & $\{0,10,20\}$\\
     $\theta$ & True ATE & 3.0\\
     \bottomrule
\end{tabular}
\label{tab:toy_details}
\end{table}

\paragraph{Model hyperparameters}
For the purposes of the toy experiment, we fix a regression neural network architecture, as well as a propensity score network architecture. The only hyperparameter we vary is $\alpha$, which is the strength of the IPM regularization term. This was done for reasons of time efficiency, as well as to have an ``apples-to-apples'' comparison between different weighting schemes used in the regression loss. The model hyperparameter values used are shown in Table \ref{tab:hyperparams}.

\begin{table}[h!]
    \centering
    \caption{Model hyperparameter ranges for toy experiment (middle column) and ``real'' datasets (IHDP/ACIC, right column). ``Wass'' is the Wasserstein distance, ``MMD-linear'' is the MMD with a linear kernel, ``MMD-RBF'' is the MMD with an RBF kernel. $e_\eta(\cdot)$ is the fully-connected neural network predicting the propensity score. ``ELU'' is the exponential linear unit activation, ``ReLU'' is the rectified linear unit activation.}
\begin{tabular}{lll}
\toprule
     \textbf{Hyperparameter} & \multicolumn{2}{c}{\textbf{Value/Range}}\\
     & Toy experiment & IHDP \& ACIC2016 \\
     \midrule
    $\alpha$ (strength of IPM term) & \{0, 0.01, 0.1, 1, 10, 100\} & $\{10^{k/2}\}_{k=-10}^6$\\
     IPM used & Wass & \{Wass, MMD-linear, MMD-RBF\}\\
     Num. hidden layers in $\Phi(\cdot)$ & 1 & \{1,2,3\}\\
     Num. hidden layers in $h(\cdot,t)$ & 1 & \{1,2,3\}\\
    Num. hidden layers in $e_\eta(\cdot)$ & 1 & \{1,2,3\}\\
     $\Phi(\cdot)$  hidden layer dim. & 100 & \{20,50,100,200\}\\
     $h(\cdot,t)$ hidden layer dim. & 100 & \{20,50,100,200\}\\
     $e_\eta(\cdot)$ hidden layer dim. & 10 & \{10,20,30\}\\
     $h(\cdot,t),\Phi$ hidden-layer activations & ELU & ELU\\
     $e_\eta(\cdot)$ hidden-layer activations & ReLU & ReLU\\
     Batch size & 200 & 200\\
     Learning rate & 0.001 & 0.001\\
     Optimizer & Adam & Adam\\
     \bottomrule
\end{tabular}
    \label{tab:hyperparams}
\end{table}
\subsection{Infant Health and Development Program (IHDP)}\label{sec:ihdp_details}
From the IHDP dataset \citep{Hill}, \cite{Shalit} made 2 datasets, named IHDP100 and IHDP1000\footnote{both datasets were downloaded from https://www.fredjo.com/}. We used the former (IHDP100) for parameter tuning/model selection, and the latter (IHDP1000) for evaluation.
For the IHDP dataset, we randomly sampled 100 hyperparameter configurations (the hyperparameter ranges are shown in Table \ref{tab:hyperparams}) -- for each sampled configuration, we train 3 models ( with respective weight schemes MW, OW, TruncIPW). We train on the IHDP100 dataset, perform early stopping based on the validation loss, and we select 3 best models (one for each weighting scheme) according to $\epsilon_{\textup{PEHE},p}^{\textup{NN}}$ on the validation set, where:
\begin{align}
    \epsilon_{\textup{PEHE},p}^{\textup{NN}} &\triangleq \frac{1}{N}\sum_{i=1}^N  [(1-2\cdot T_i)(Y_{j(i)}-Y_i)-(h(\Phi(X_i,1)-h(\Phi(X_i,0)]^2\\
    j(i) &\triangleq \underset{j: T_j = 1-T_i}{\textup{argmin}}||X_i-X_j||_2
\end{align}
This is a proxy for $\epsilon_{\textup{PEHE},p}$ which does not make use of counterfactual information.
After the model tuning stage on IHDP100, we report 3 results (1 for each weight scheme) on the IHDP1000 dataset.

For the causal forest results in Section 4.3 of the main text, we obtained the representations and weights (obtained from our 3 best models) for IHDP100, and used them as input to a causal forest (CF) algorithm. We then compared the augmented CF models to a vanilla CF model on IHDP100. More details are provided in the section below.
\subsection{Causal Forests}
\paragraph{IHDP100 weight ablation}
In addition to comparing the vanilla CF with the CF augmented with learned weights and representations, we add a comparison to the CF augmented with the representations only (\textit{i.e.}, without weights). We find that the unweighted augmented CF (``CF+$\Phi$'' in Table \ref{tab:causal_forest_ihdp}) performs similarly to its weighted counterparts for the IHDP100 dataset.
\begin{table}[!h]
    \vspace{-1mm}
    \centering
    \caption{Causal forest (CF) results for IHDP100. The top block is a vanilla CF model. The middle block is a causal forest model using learned representations (denoted $\Phi$) without weights (\textit{i.e.}, the equivalent of CFRNet). The bottom block consists of causal forest models using the learned representations and weights. The bottom block rows are the weights used in the training objective and as the per-sample weights to train the CF.}
    \begin{tabular}{@{}lcc@{}}
    \toprule
     &  $\sqrt{\epsilon_{\textup{PEHE},p}}$ &  $\epsilon_{\textup{ATE},p}$\\
     \midrule
     CF & 3.54 $\pm$ .58 & .47 $\pm$ .06\\
     \midrule
     CF + $\Phi$ & 1.52 $\pm$ .35 & $.20 \pm .04$\\
     \midrule
     CF + $\Phi$ + MW & 1.51 $\pm$ .31 & .20 $\pm$ .03\\
     CF + $\Phi$ + OW & 1.59 $\pm$ .31 & .19 $\pm$ .03\\
     CF + $\Phi$ + TruncIPW & 1.55 $\pm$ .35 & .22 $\pm$ .03\\   \bottomrule
    \end{tabular}
    \vspace{-4mm}
    \label{tab:causal_forest_ihdp}
\end{table}
\paragraph{Atlantic Causal Inference Competition 2016 (ACIC2016)}
In Section 4.3 of the main text, we considered the ACIC2016 dataset \citep{dorie2019}, which comprises 77 datasets (we use 10 repetitions of each)\footnote{generated using https://github.com/vdorie/aciccomp/tree/master/2016, and setting \texttt{parameterNum} between 1 and 77, and \texttt{simulationNum} between 1 and 10.}, each with 4802 samples, and 58-dimensional covariates. ACIC2016 uses the same covariates for all the datasets, but different data-generating mechanisms for potential outcomes and treatment across datasets. We removed the categorical covariates (named $x_2,x_{21},x_{24}$ in the dataset), since our models are not equipped to handle categorical data. We standardized the remaining 55 covariate dimensions (\textit{i.e.}, for each dimension we subtract the mean and divide by the standard deviation). We used the first 4000 samples for training, and the remaining 802 for testing. We used 30\% of the training set for validation.

The tuning procedure is similar to the one described for IHDP. We use the first 10 (out of 77) datasets, with 1 repetition of each, as a tuning set. We pick 3 best models (one for each weighting scheme) according to the average $\epsilon_{\textup{PEHE},p}^{\textup{NN}}$ (across the 10 datasets) on the validation set. The hyperparameter ranges used for tuning are shown in Table \ref{tab:hyperparams}. Hence, after tuning, we have 3 ``best'' models (1 for each weighting scheme), which we apply to the 77 datasets (with 10 repetitions).

We then use the obtained representations and weights as input to a causal forest (CF) model\footnote{Implemented using the \texttt{causal\_forest} function of the \texttt{GRF} package in R: https://CRAN.R-project.org/package=grf}, and we compare the performance of the ``vanilla'' causal forest (\textit{i.e.}, CF using only the original covariates) with the performance of the ``augmented'' CF models (\textit{i.e.}, which use the learned representations and weights as input). More specifically, for the augmented models, we use the learned representations as the ``covariates'', and we use the propensity-based weights as the per-sample weights to train the CF.
\section{Additional Results}

\subsection{Toy experiment}
\paragraph{Performance on target populations $g(x)$}
In Section 4.1 of the main text, we measured the performance of our models on the observed population $p(x)$. Here, we extend this evaluation to the different target populations $g(x)$. We report performance using $\sqrt{\epsilon_{\textup{PEHE},g}}$, computed via equation (5) from the main text. Figure \ref{fig:flipped_weighted_pehe} shows a plot of $\epsilon_{\textup{PEHE},g}$ for all choices of $g(x)$, for all toy datasets, and for all weight schemes used during training. 
From Figure \ref{fig:flipped_weighted_pehe}, we can see that each weighting scheme (row) tends to do well for the target population it was trained for -- \textit{i.e.}, within each row, the corresponding metric (color) is lowest (or close to the lowest). This provides some evidence for the fact that the models perform well on the target population they were trained for.
\begin{figure}
    \centering
    \includegraphics[width=\textwidth]{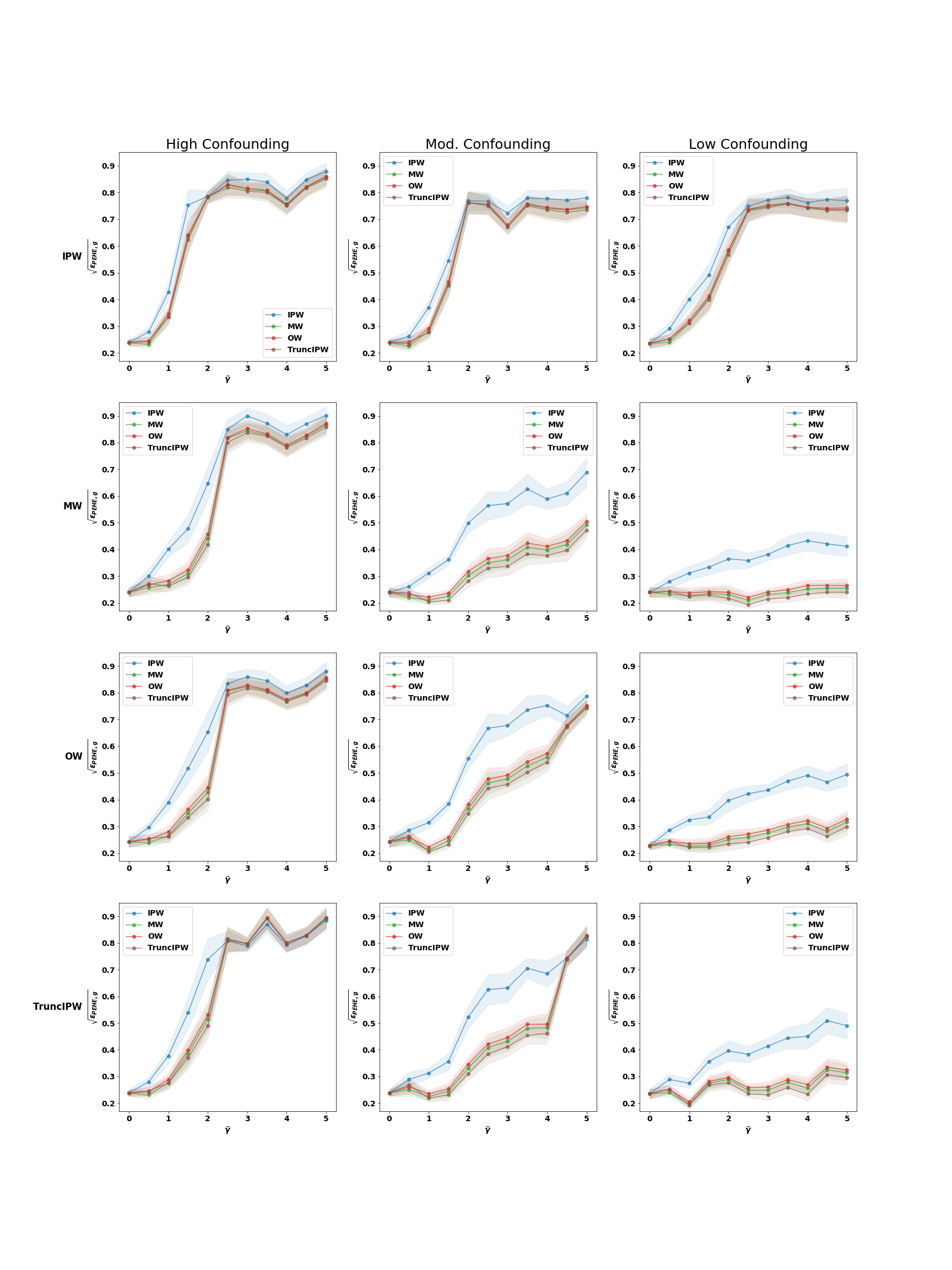}
    \caption{$\epsilon_{\textup{PEHE},g}$ vs. $\Tilde{\gamma}$. The colors are the different choices of $g(x)$ (to calculate $\sqrt{\epsilon_{\textup{PEHE},g}}$ from equation (5) in the main text), and the rows are the weight schemes used during training (\textit{i.e.}, in equation (11) of the main text).}
    \label{fig:flipped_weighted_pehe}
\end{figure}

\paragraph{Double-Robust ATE estimation}
We may easily enhance the ATE estimate in equation (3) of the main text by accounting for model bias, using equation (6) from \cite{Mao2018}. Specifically, we can define biases as:
\begin{align}
    b^{(t)} = \frac{1}{\sum_{i: T_i=t} w_\eta(X_i,t)}\sum_{i:T_i=t }w_\eta(X_i,t)[h(\Phi(X_i),t)-Y_i], \hspace{0.1in} \text{for}~~t \in \{0,1\}
\end{align}

which we may use to obtain a doubly-robust \citep{Lunceford} ATE estimate via:
\begin{align}\label{eq:ate_dr}
    \hat{\tau}_{\textup{ATE},g}^{\textup{DR}} = \hat{\tau}_{\textup{ATE},g} -b^{(1)}+b^{(0)}
\end{align}
Note that $b^{(1)},b^{(0)}$ are calculated using the training set only.
We compute the target population ATE error via:
\begin{align}
    \epsilon^{\textup{DR}}_{\textup{ATE},g} = |\tau_{\textup{ATE},g} - \hat{\tau}_{\textup{ATE},g}^{\textup{DR}}|
\end{align}

Figure \ref{fig:ate_dr_diff_clean} shows percent improvement of the double-robust estimator, computed as
\begin{align}
\Delta_g^{\textup{DR}} \triangleq \frac{\epsilon_{\textup{ATE},g}-\epsilon_{\textup{ATE},g}^{\textup{DR}}}{\epsilon_{\textup{ATE},g}} \label{eq:percent_imp_DR}
\end{align}
 on the (test) toy datasets. The double-robust ATE estimator (i.e. $\hat{\tau}_{\textup{ATE},g}^{\textup{DR}}$) enjoys some improvement in the ATE estimation error in most cases, though this is not true across all the toy datasets.

\begin{figure}[!h]
    \centering
    \makebox[\textwidth][c]{\includegraphics[width=1.3\textwidth]{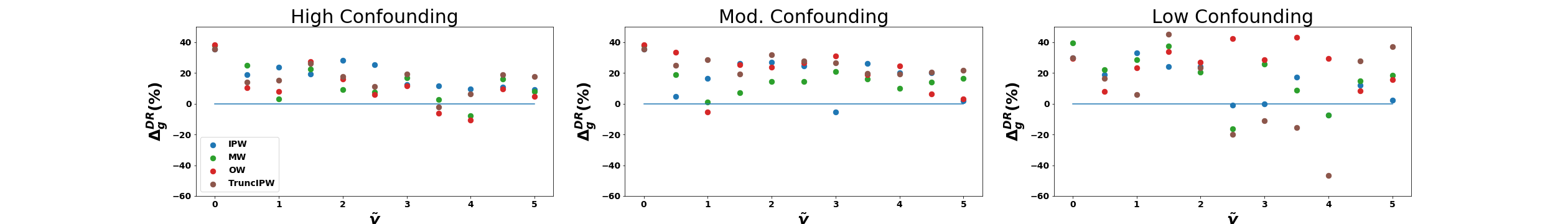}}
    \caption{Improvement in ATE estimation (on the toy datasets) using the double-robust $\hat{\tau}_{\textup{ATE},g}^{\textup{DR}}$ from \eqref{eq:ate_dr}. The x-axis is the imbalance parameter $\Tilde{\gamma}$, the y-axis is the percent improvement defined in \eqref{eq:percent_imp_DR}. The colors are different choices of $g$ (used for (i) the weight schemes during training, and (ii) to compute the target population metric $\epsilon_{\textup{ATE},g}$. Positive values means the double-robust estimator $\hat{\tau}_{\textup{ATE},g}^{\textup{DR}}$ improves upon the vanilla estimator $\hat{\tau}_{\textup{ATE},g}$.}
    \label{fig:ate_dr_diff_clean}
\end{figure}
\subsection{IHDP100 additional comparisons}
In Table 2 of the main text, some of the listed methods (namely, RCFR and CFR-ISW) actually reported their results on IHDP100, whereas we reported performance of our methods on IHDP1000. For the sake of completeness, we add the comparison between our methods, RCFR, and CFR-ISW on the IHDP100 dataset, reported in the table below. These results are consistent with the results from Table 2 in the main text, namely that our proposed methods perform on-par with state-of-the-art methods from recent work.

\begin{table}[!h]
    \caption{Results on IHDP100 test set. The top block consists of baselines from recent work. The bottom block is our proposed methods. Lower is better.}
    \centering
    \begin{small}
    \begin{tabular}{@{}lcc@{}}
    \toprule
     Model &  $\sqrt{\epsilon_{\textup{PEHE},p}}$ &   $\epsilon_{\textup{ATE},p}$\\
     \midrule
        CFR-ISW \citep{hassanpour2019counterfactual} & $.70 \pm .1$ & $.19 \pm .03$ \\
        RCFR \citep{johansson2018learning} & $.67 \pm .05$ & -  \\
        \midrule
        BWCFR-MW (Ours)&  .66 $\pm$ .06 &  .18 $\pm$ .02 \\
        BWCFR-OW (Ours)&   .66 $\pm$ .06 &  .16 $\pm$ .02 \\
        BWCFR-TruncIPW (Ours)&  .65 $\pm$ .05 &  .16 $\pm$ .02 \\
        \bottomrule
    \end{tabular}
    \end{small}
    \label{tab:IHDP_results_SM}
\end{table}

\subsection{Atlantic Causal Inference Competition 2016 (ACIC2016)}
In this section, we carefully examine the results of Section 4.3 from the main text. Specifically, Figure \ref{fig:ACIC} shows the performance of the proposed methods on each of the 77 datasets in ACIC2016, rather than an aggregate as shown in Table 3 of the main text.
Figure \ref{fig:ACIC} compares the performance of 3 types of models:
\begin{itemize}
    \item A vanilla causal forest algorithm
    \item Our proposed deep methods
    \item A hybrid model consisting of a causal forest augmented with our learned representations and weights (obtained from $\Phi(x)$ and $w_\eta(x,t)$, respectively).
\end{itemize}

From Figure \ref{fig:ACIC}, we can see that the augmented causal forest consistently outperforms the vanilla causal forests for almost all of the 77 datasets, both in terms of  $\sqrt{\epsilon_{\textup{PEHE},p}}$ and $\epsilon_{\textup{ATE},p}$. The neural network models perform (on average) better than the vanilla CF in terms of $\sqrt{\epsilon_{\textup{PEHE},p}}$, but worse than the augmented CF. In terms of $\epsilon_{\textup{ATE},p}$, the neural network models perform worse. A possible explanation for the poor performance of the neural network is that our tuning procedure was conducted only on the first 10 datasets (with 1 repetition each), whereas the ACIC2016 set comprises 77 datasets with 10 repetitions each.
This suggests that we may stand to benefit from using hybrid approaches for ITE estimation, since the hybrid approach outperforms the 2 individual components it is comprised of, even though the deep models were not extensively tuned.
Further, the hybrid models trained with Overlap Weights performed the best.
\begin{figure}
    \centering
     \makebox[\textwidth][c]{\includegraphics[width=\textwidth]{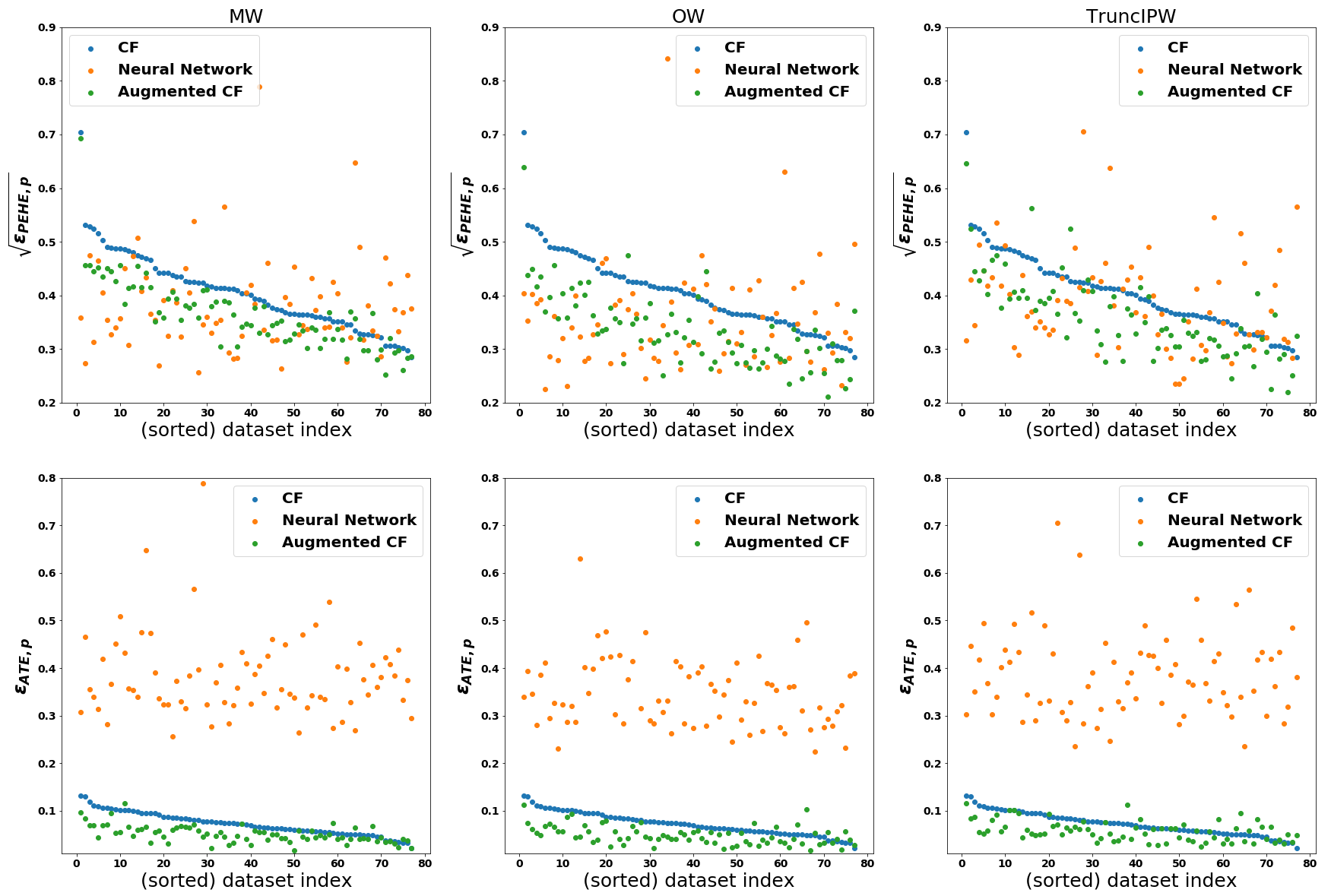}}
    \caption{Per-dataset results on ACIC2016. The first row shows $\sqrt{\epsilon_{\textup{PEHE},p}}$, and the second row shows $\epsilon_{\textup{ATE},p}$. ``CF'' is the vanilla causal forest, ``Augmented CF'' is the CF trained using the learned representations and weights, and ``Neural Network'' are our proposed methods. The columns are the weight schemes used to (i) obtain the representations and weights for the augmented CF, and (ii) used to train the neural network models. The datasets were sorted in descending order according to the performance of the causal forest.}
    \label{fig:ACIC}
\end{figure}
\section{Computing infrastructure and details}
All computation was done using Python and R. All neural network models were created and trained using Tensorflow 1.13.1 \citep{tensorflow}. Computations were done on an NVIDIA Geforce GTX 1080 Ti. The reported results on IHDP1000, ACIC2016, and the toy dataset each took approximately 20 hours to run.
\newpage
\bibliography{main}
\bibliographystyle{bib_style}
\end{document}